\documentclass{article}
\usepackage{fullpage}

\usepackage{microtype}
\usepackage{graphicx}
\usepackage{subfig}
\usepackage{booktabs} 
\usepackage{amsmath}
\usepackage{amssymb}
\usepackage{amsthm}
\usepackage{bbold}
\usepackage{xspace}
\usepackage{comment}
\usepackage{ifthen}
\usepackage{adjustbox}
\usepackage{tabularx}
\usepackage{multirow}
\usepackage{xcolor}
\usepackage{enumitem}
\usepackage{tablefootnote}

\usepackage{array}
\newcolumntype{P}[1]{>{\centering\arraybackslash}p{#1}}

\newcolumntype{b}{>{\hsize=1.3\hsize}X}
\newcolumntype{s}{>{\hsize=.45\hsize}X}
\newcolumntype{m}{>{\hsize=.9\hsize}X}

\newtheorem{theorem}{Theorem}[section]
\newtheorem{lemma}[theorem]{Lemma}

\newtheorem{fact}[theorem]{Fact}

\newtheorem{definition}[theorem]{Definition}

\usepackage{hyperref}

\usepackage[capitalize, nameinlink]{cleveref}

\newcommand{\msr}{{\large$^\star$}}
\newcommand{\ms}{{\large$^\dagger$}}
\newcommand{\aws}{{\large$^\ddagger$}}

\makeatletter

\newcommand{\note}[2]{
    \ifthenelse{\equal{\showComments}{yes}}{\textcolor{#1}{#2}}{}
}

\newcommand{\todo}[1]{{\color{red} \sf \bf [TODO: #1]}}

\newcommand{\showComments}{yes}

\newcommand{\maxfpgasize}{\ensuremath{M}}
\newcommand{\acc}{\ensuremath{\mathrm{acc}}}
\newcommand{\cpu}[1]{p_{#1}^{\mathrm{cpu}}}
\newcommand{\cpunoindex}{\mathrm{cpu}}
\newcommand{\fpga}[1]{p_{#1}^{\acc}}
\newcommand{\fpganoindex}{\acc}
\newcommand{\mem}[1]{m_{#1}}
\newcommand{\comm}[1]{c_{#1}}
\newcommand{\singlesubgraphcount}{q}
\newcommand{\subgraphcount}{kq}

\newcommand{\TotalLatency}{\mathrm{TotalLatency}}
\newcommand{\Latency}[1]{\mathrm{Latency}_{#1}}
\newcommand{\Incoming}[1]{\mathrm{CommIn}_{#1}}
\newcommand{\Outgoing}[1]{\mathrm{CommOut}_{#1}}
\newcommand{\Start}[1]{\mathrm{SubgraphStart}_{#1}}
\newcommand{\Finish}[1]{\mathrm{SubgraphFinish}_{#1}}

\newcommand{\MaxLoad}{\mathrm{MaxLoad}}
\newcommand{\Load}[1]{\mathrm{Load}_{#1}}

\newcommand{\numideals}{\mathcal{I}}

\newcommand{\ncmaxexpert}{$2\times$\xspace}
\newcommand{\ncavgexpert}{$1.46\times$\xspace}
\newcommand{\ncmaxlocal}{$2.08\times$\xspace}
\newcommand{\ncavglocal}{$1.29\times$\xspace}
\newcommand{\ncmaxpd}{$1.21\times$\xspace}
\newcommand{\ncavgpd}{$1.10\times$\xspace}
\newcommand{\ncmaxscotch}{$7.69\times$\xspace}
\newcommand{\ncavgscotch}{$1.50\times$\xspace}

\newcommand{\expertnc}{$68\%$\xspace}
\newcommand{\expertc}{$71\%$\xspace}

\title{Efficient Algorithms for Device Placement\\ of DNN Graph Operators}

\date{}

\author{
Jakub Tarnawski\msr, Amar Phanishayee\msr, Nikhil R. Devanur\aws\thanks{Work done while at Microsoft Research.},\\
Divya Mahajan\ms, Fanny Nina Paravecino\ms
\\
\\
\rm{\textit{\msr Microsoft Research \hspace{0.02in} \aws Amazon \ms Microsoft\hspace{0.02in}}}
}

\begin{document}

\maketitle

\begin{abstract}
Modern machine learning workloads use large models, with complex structures, that are very expensive to execute. The devices that execute complex models are becoming increasingly heterogeneous as we see a flourishing of domain-specific accelerators being offered as hardware accelerators in addition to CPUs. These trends necessitate distributing the workload across multiple devices. Recent work has shown that significant gains can be obtained with \emph{model parallelism}, i.e, partitioning a neural network's computational graph onto multiple devices. In particular, this form of parallelism assumes a \emph{pipeline} of devices, which is fed a stream of samples and yields high throughput for training and inference of DNNs. However, for such settings (large models and multiple heterogeneous devices), we require automated algorithms and toolchains that can partition the ML workload across devices.
In this paper, we identify and isolate the \emph{structured optimization problem} at the core of device placement of DNN operators, for both inference and training, especially in modern pipelined settings. We then provide algorithms that solve this problem to optimality. We demonstrate the applicability and efficiency of our approaches using several contemporary DNN computation graphs.
\end{abstract}
\section{Introduction}
\label{sec:intro}

Deep Neural Networks (DNNs) have been effective across a range of applications, including image classification~\cite{krizhevsky2012imagenet,simonyan2014very,he2016deep}, translation~\cite{wu2016google}, language modeling~\cite{merity2017regularizing}, and video captioning~\cite{venugopalan2015sequence}.
The proliferation of heterogeneous hardware accelerators~\cite{tpu, dnnweaver} coupled with the dramatic growth in the size and the structural complexity of DNNs has bolstered the importance of \textit{model parallelism}, where for both inference and training, the model is distributed across devices.

\textbf{DNN inference} in the ``single-stream'' setting~\cite{mlperf-inf}, where only one inference request is issued at a time, is \emph{latency-sensitive}.
To achieve low latency, model parallel executions split the model across many  accelerators~\cite{eyeriss-isscc-2016, bw-isca-2018, bw-micro-2018}.
Model-parallel inference is beneficial due to three primary reasons.
First, such splits are mandated by the memory-capacity (size) limitations of accelerators that cannot fit a single DNN model.
Current DNN models have billions of parameters and require multiple GBs of space to store the weights and intermediate activations.
Second, wide branching in recent DNN structures, as well as in the operator-granularity graphs for established DNNs, opens up the potential of executing data-independent sections of the computation in parallel to reduce latency.
Third, the model needs to be split across multiple types of devices when a subset of operators in the graph are better suited or only supported to execute on certain devices.

\textbf{DNN training}, on the other hand, is \emph{throughput-bound}, as is DNN inference for the ``offline'' setting where many inputs can be serviced together~\cite{mlperf-inf}.
Model parallelism has been proposed for training for the very same motivational reasons listed for inference above~\cite{krizhevsky2012imagenet, krizhevsky2014one}. 
Early influential systems such as DistBelief~\cite{dean2012distbelief} and Project Adam~\cite{chilimbi2014adam} split models to operate on commodity CPU clusters and out of CPU caches.
In such a setting, operators in a DNN model are partitioned across the available devices, with each device evaluating and performing updates only on a subset of the model’s parameters for all inputs.
While traditional model parallel training suffers from problems of low hardware utilization, as only a single accelerator is active at any given time, \textbf{pipelined model parallelism} overcomes this deficiency.
The amount of data communicated in pipelined training is the size of intermediate outputs (and corresponding gradients), which need to be sent across accelerators, and is much lower than the size of data communicated in data-parallel training.
In particular, for a range of existing models that fit on a single GPU, PipeDream~\cite{harlap2018pipedream, narayanan2018pipedream}
uses pipelined model-parallelism to achieve much faster training time to advertised accuracy than data-parallelism.
Similarly, GPipe~\cite{huang2018gpipe, huang2019gpipe}
uses pipelined model-parallel training for very large models whose total training memory footprint exceeds the memory capacity of a single accelerator.

Given the importance of model-parallel inference and training, in this paper we present efficient algorithms to answer the following general question: \textit{For a DNN model and a deployment scenario (a set of accelerators and their memory and interconnect constraints), how can we effectively partition the model to optimize the metric of interest, such as latency or throughput, relevant to the inference or training task at hand?}

We provide novel algorithmic approaches to tackle the problem of partitioning the model in both model-parallel
inference and training scenarios, optimizing for their corresponding metrics of interest:
\begin{itemize}
    \item Inference \--- (i) Model-Parallel Inference, optimized for ``single-stream'' latency (\cref{fig:inf_mp}), (ii) Pipelined Inference, optimized for ``offline'' throughput (\cref{fig:inf_pipe}).
    \item Training, optimized for throughput \--- (i) Model-Parallel Training (\cref{fig:train_mp}), (ii) Pipeline-Parallel Training with PipeDream and GPipe schedules (\cref{fig:training_schedule}).
\end{itemize}

In particular, for both non-pipelined and pipelined settings, we identify the combinatorial optimization problem at the core of the device placement question, whose solution will yield the \emph{optimal} partition.
We then show how to solve this problem to optimality via Integer Programming (IP) and Dynamic Programming (DP) based algorithms.
Our methods are general as they can be applied either to coarse-granularity layer graphs or to more complex fine-granularity operator graphs.
We support graph partitions where accelerators can hold a \emph{non-contiguous} fragment of the graph.
We evaluate our partitioning algorithms for different scenarios described above for a variety of modern DNN workloads (7 DNNs, 16 layer and operator graphs).
We find that the placements are efficient and result in non-trivial optimal splits; non-contiguous splits outperform all the techniques, with an improvement of up to \ncmaxexpert over expert (average \ncavgexpert), \ncmaxlocal over local search (average \ncavglocal)~\cite{localsearch}, \ncmaxpd over PipeDream (average \ncavgpd)~\cite{narayanan2018pipedream}, \ncmaxscotch over Scotch (average \ncavgscotch)~\cite{scotch}.
%

\paragraph{Outline.}
This paper is organized as follows.
We discuss related work in \cref{sec:related_work}.
In \cref{sec:model} we define our model of DNN computation,
as well as the input/output specification for the problem of finding the best split.
Next, in \cref{sec:latency,sec:throughput} we focus on the latency objective (feeding one sample) and throughput objective (pipelining multiple samples), respectively.
We present efficient algorithms that find optimal splits for both objectives.
In \cref{sec:experiments,sec:latency_experiments} we present our evaluation results
for the throughput objective and the latency objective, respectively.

\section{Related Work}
\label{sec:related_work}

In the context of DNN workloads, model partitioning across different devices
has mostly been a manual process driven by human experts.
Most prior work on \emph{automated} device placement falls into two broad categories.

The first category comprises methods that treat the objective function (i.e., latency or throughput) as a black box.
These works use heuristics, mostly based on reinforcement learning, to find partitions for a given workload (Mirhoseini et al.~\cite{dean2017rlplacement, hierarchical-2018}, Spotlight~\cite{spotlight}) or learn a placement policy that can then be adjusted for new workloads via transfer learning (Placeto~\cite{placeto}, GDP~\cite{gdp}) or used to bootstrap a genetic algorithm (REGAL~\cite{regal}).
Unfortunately, these methods are computationally expensive, as they need to evaluate large numbers of placements, each of which entails a reconfiguration of the deployed devices (for a new DNN split) and measuring the runtime of several inference/training steps.
For instance, ~\cite{dean2017rlplacement} requires 12--27 hours of training time \emph{on the target system} to partition modern workloads; \cite{hierarchical-2018} requires 12 GPU hours.
For this reason, some systems (Placeto~\cite{placeto}, FlexFlow~\cite{jia2018beyond}) resort to implementing a simulator to evaluate the objective.
%
%

Works in the second category -- including ours -- build a cost model that closely reflects real performance, and then algorithmically solve the resulting ``offline'' optimization problem of finding good partitions and schedules.
This includes classic results in scheduling on multiple machines/devices
\cite{lawler1993sequencing,GrahamListScheduling1966,KernighanLin,PapadimitriouY90,SkutellaW99,ShmoysT93}, as well as modern DNN scheduling works (OptCNN~\cite{jia2018exploring}, PipeDream's~\cite{narayanan2018pipedream} optimizer).
Such algorithms use \emph{profiled} compute time of each node (layer or operator) and data-transfer requirements between nodes in a graph, and the target deployment system infrastructure such as machine and network properties (e.g.~measured bandwidths).
Such techniques do not evaluate the performance of splits in an online fashion.
Nevertheless, it has been demonstrated that for well-defined cost models the objective function closely matches real performance (PipeDream~\cite[Figure 15]{narayanan2018pipedream}, FlexFlow~\cite[Figure 11]{jia2018beyond}, OptCNN~\cite[Table 4]{jia2018exploring}).
Our throughput maximization model in \cref{sec:throughput} generalizes the cost model used in PipeDream~\cite{narayanan2018pipedream}, and our latency minimization objective (\cref{sec:latency}) is similar to the cost model of FlexFlow's simulator~\cite{jia2018beyond}.
In terms of approach, both OptCNN~\cite{jia2018exploring} and FlexFlow~\cite{jia2018beyond} optimize over different dimensions than our methods, opting for more local parallelization strategies.

\paragraph{Pipelining.}
GPipe~\cite{huang2019gpipe} and PipeDream~\cite{narayanan2018pipedream} introduce \emph{pipelined} model-parallelism for training.
Given that this prior work has already shown the efficacy of pipeline parallel training on statistical efficiency (training progress compared to data-parallel training), the focus of this paper is instead on efficient algorithms to effectively partition DNN models across accelerators.
For finding good DNN splits, GPipe presents no algorithm, and PipeDream proposes a method limited to layer graphs that are linear (i.e., a path).
Efficiently finding optimal splits for pipelined execution in a general-DAG setting for both training and inference is a central contribution of this paper.

\section{Computational Model}
\label{sec:model}

\paragraph{Input.}
We consider a heterogeneous system with $k$ DNN hardware accelerators
and $\ell$ CPUs. 
For simplicity of exposition we assume all accelerators to be of the same type (such as GPU, FPGA, or TPU) for a single input.
Every such accelerator has a capacity limit for its associated memory, denoted by $\maxfpgasize$.
We refer to both CPUs and accelerators as \emph{devices}.
The rest of the input to our algorithms consists of a directed acyclic graph (DAG) $G = (V,E)$ with associated weights:

\begin{itemize}
    \item The set $V$ of nodes represents operators such as \texttt{MatMul}, \texttt{Add}, \texttt{ReLu}, etc. (for operator graphs), or layers such as \texttt{MaxPool2d} or \texttt{LSTM} (for layer graphs).
    Each node $v$ has an associated time $\cpu{v}$ required to process $v$ on a CPU, as well as the processing time $\fpga{v}$ of $v$ on an accelerator.\footnote{If $v$ is not supported on the accelerator, we set $\fpga{v} = \infty$.}
    Each node also has a size $\mem{v}$: the memory usage of its associated weights and activations.
    
    \item The set $E$ of directed edges encodes dependency/precedence constraints: an edge $(u,v)$ implies that the operation $v$ depends on the result of $u$. 
    Each node $u$ has a communication cost $\comm{u}$, which corresponds to the time required to transfer $u$'s output between CPU DRAM
    (henceforth referred to as RAM)
    and the accelerator's memory, say through a PCIE bus.
    Crucially, this cost is paid only if $u$ and $v$ are placed on different devices: if $u$ is on an accelerator, it needs to write this output to RAM, and if $v$ is on an accelerator, it needs to read this input from RAM.
    We ignore the cost of reading or writing to RAM from CPUs.
\end{itemize}

\paragraph{Output.}
We seek to assign each node in the graph to exactly one device so that for every accelerator the sum of sizes $\mem{v}$ of nodes assigned to it does not exceed its capacity~$\maxfpgasize$.
Out of all feasible partitions we want to select one that optimizes a metric of interest (latency or throughput).

\begin{figure}
\centering
\subfloat[Contiguous split]{
\includegraphics[width=0.46\textwidth]{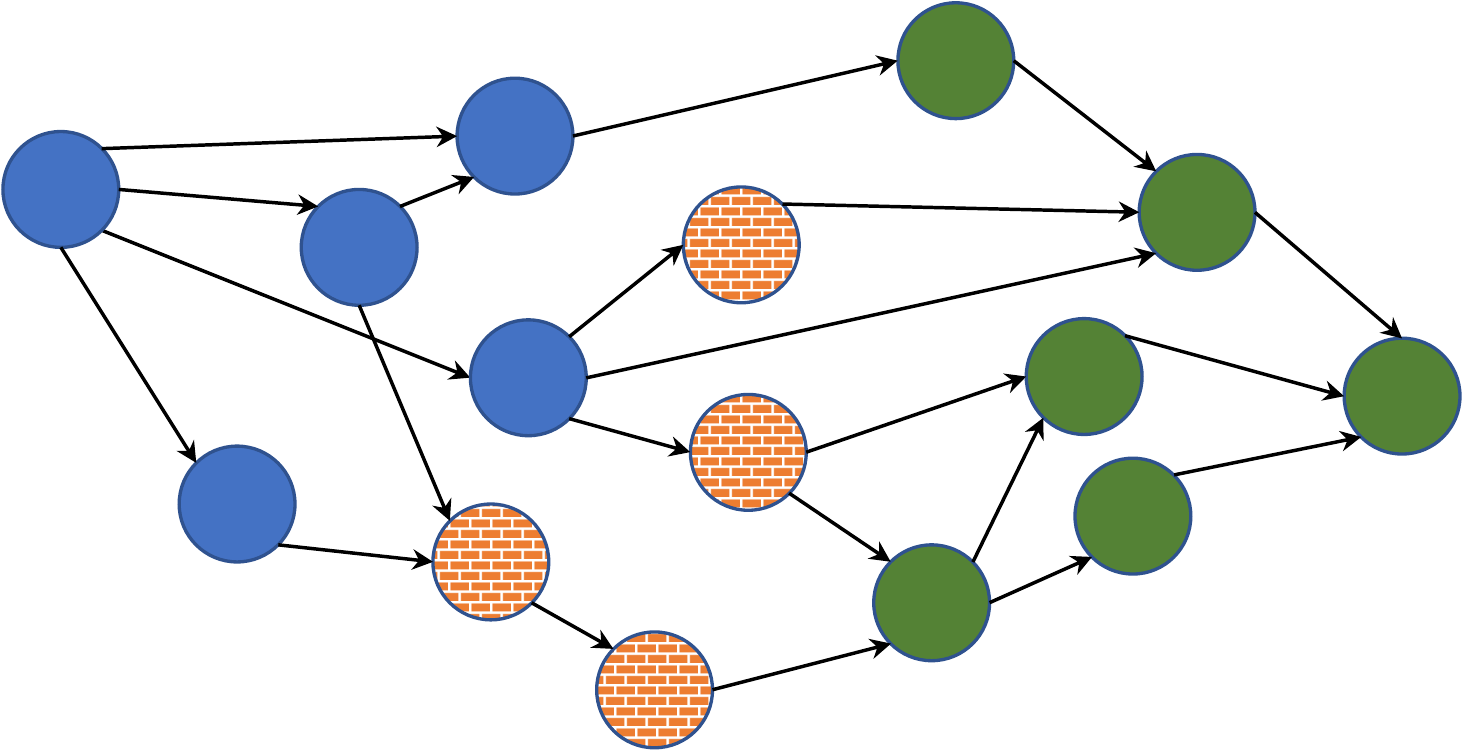}
\label{fig:contig_split}}
\subfloat[Non-contiguous split]{
\includegraphics[width=0.46\textwidth]{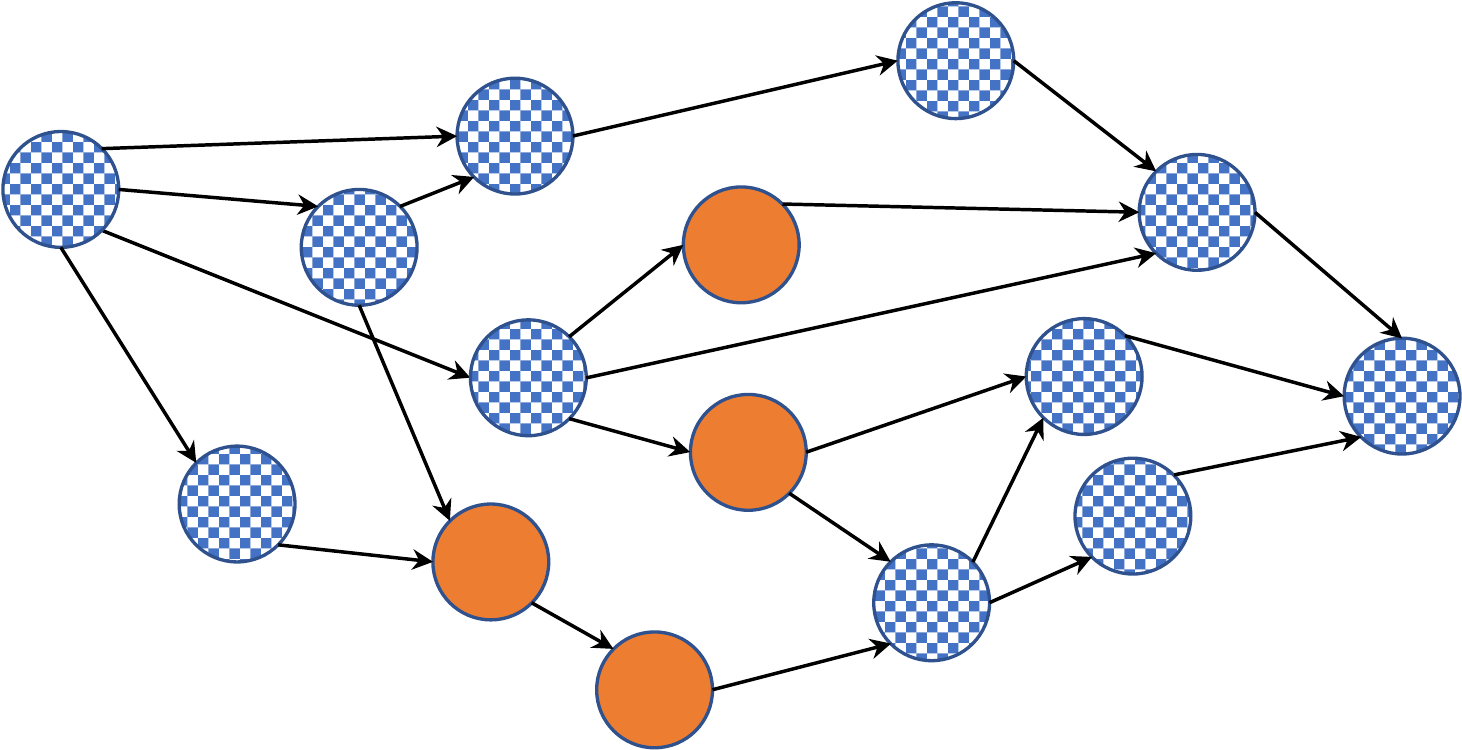}
\label{fig:ncontig_split}}
\caption{(a) Contiguous and (b) non-contiguous splits. Note that the brick-patterned orange nodes in~(a) form a contiguous subgraph despite not being connected, and checked blue nodes in~(b) form a non-contiguous subgraph despite being connected.}
\label{fig:contiguous}
\end{figure}

\paragraph{Contiguous and non-contiguous subgraphs of computation.}
By default, we desire every device to hold a contiguous fragment of the DNN:
\begin{definition}
\label{def:contiguous}
    We say that a set $S \subseteq V$ is \emph{contiguous} if there do \textbf{not} exist nodes $u \in S$, $v \in V \setminus S$, and $w \in S$ such that $v$ is reachable from $u$ and $w$ is reachable from $v$.
\end{definition}

See \cref{fig:contiguous} for an example.
This property enables subgraphs to be invoked in an uninterrupted way:
all required inputs can be transferred to the accelerator at one time,
after which it performs computations and produces all its outputs.
This allows for simpler system implementations and less interactivity
with the accelerator.\footnote{
    In particular,
    this way of invoking subgraphs of computation on accelerators is motivated by production systems at Microsoft~\cite{bw-isca-2018,bw-micro-2018}, where there is no state maintained {across} any two subgraph invocations other than subgraph model parameters.
}

However, in this work we also explore non-contiguous splits, where the subgraphs placed on an accelerator can be arbitrary.
In particular, we explain how to build a pipelined schedule for executing such a split for a stream of many samples,
and how to find an optimal split of this more general form.

\section{Inference and Latency Minimization}
\label{sec:latency}

\begin{figure}
\centering
\subfloat[Model-parallel inference schedule]{
\includegraphics[width=0.5\textwidth]{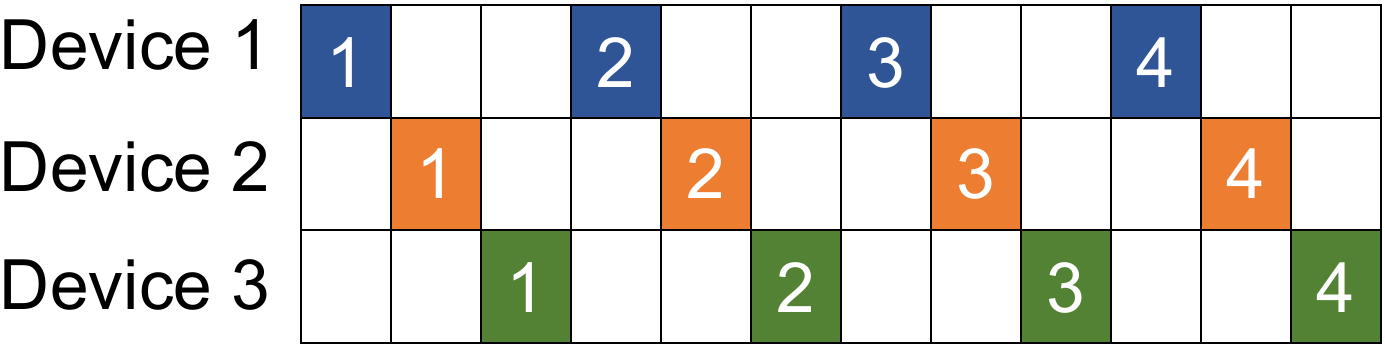}
\label{fig:inf_mp}}
\subfloat[Model-parallel training schedule]{
\includegraphics[width=0.5\textwidth]{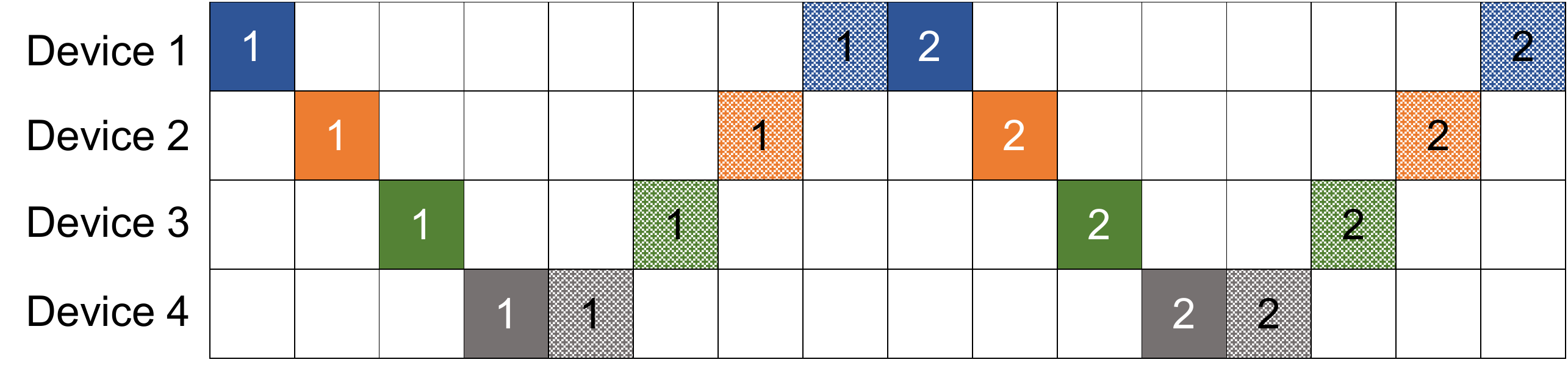}
\label{fig:train_mp}}
\caption{(a) Single-stream model-parallel inference and (b) Model-parallel training schedule with darker shades for forward pass and lighter for  backward. The $x$-axis is time, and numbers 1--4 are input minibatch identifiers. Different colors represent different devices.}
\label{fig:mp}
\end{figure}

In this section we focus on the task of DNN inference in the non-pipelined setting, i.e., when one sample is fed at a time (see \cref{fig:mp}).
The objective here is latency, i.e., the time to produce the final output.
Here, model parallelism is required and/or assists in the following  two ways.
First, the model might not fit in the memory of a single accelerator,
making the split necessary.
Second, it enables us to exploit the parallelism inherent in the model: if two operators are independent, they can be processed simultaneously if placed on different devices.
We present an \emph{Integer Programming} based solution for this setting.

\paragraph{Mode of computation.}
In the setting of latency minimization,
we assume the way of invoking contiguous subgraphs of computation
that we mentioned above (in \cref{sec:model}).
Specifically, an accelerator, which is assigned a subgraph $S \subseteq V$ of nodes, can be \emph{invoked} when all of its required inputs are ready in DRAM
(these are outputs of nodes not in $S$ but with an edge to $S$).
Once invoked, the accelerator transfers this data to its memory.
Next, it processes operations $v \in S$ (in some sequential order).
Finally, it transfers the results back to DRAM (these are outputs of nodes in $S$ with an edge leaving $S$).
This uninterrupted mode of execution is made possible by $S$ being contiguous.

Another mild assumption we make to streamline the Integer Programming formulation is that the number $\ell$ of CPU cores is no smaller than the \emph{width} of $G$, i.e., the maximum number of nodes that can feasibly be processed in parallel.\footnote{Formally, $\ell$ is larger than any \emph{antichain}: a set $A \subseteq V$ of nodes such that for any $u,v \in A$, $u$ is not reachable from $v$.}

\paragraph{Our formulation.}

\begin{figure*}[t]
    \begin{alignat}{4}
        \text{min} \quad && \TotalLatency \nonumber \\
        \text{s.t.} \quad && \sum\nolimits_{i=0}^{k} x_{vi} \ & = 1 & \quad & (\forall v) \label{con:assigned} \\
        && \text{subgraph } &\{ v \in V : x_{vi} = 1 \}  \text{ is contiguous} && (\forall i=1,...,k) \label{con:contiguous} \\
        && \maxfpgasize & \ge \sum\nolimits_v \mem{v} \cdot x_{vi} && (\forall i=1,...,k) \label{con:memory} \\
        && \Incoming{ui} & \ge x_{vi} - x_{ui} && (\forall (u,v) \in E) \ (\forall i=1,...,k) \label{con:incoming} \\
        && \Outgoing{ui} & \ge x_{ui} - x_{vi} && (\forall (u,v) \in E) \ (\forall i=1,...,k) \label{con:outgoing} \\
        && \TotalLatency & \ge \Latency{v} && (\forall v) \nonumber \\
        && \Start{i} & \ge \Latency{v} \cdot \Incoming{vi} && (\forall v) \ (\forall i=1,...,k) \label{con:latency_incoming} \\
        && \Finish{i} & = \Start{i} 
        + \sum\nolimits_v \Incoming{vi} \cdot \comm{v} \nonumber \\
        &&& + \sum\nolimits_v x_{vi} \cdot \fpga{v}
        + \sum\nolimits_v \Outgoing{vi} \cdot \comm{v} && (\forall i=1,...,k) \label{con:starttofinish} \\
        && \Latency{v} & \ge x_{v0} \cdot \cpu{v} && (\forall v) \label{con:latency1} \\
        && \Latency{v} & \ge x_{v0} \cdot \cpu{v} + \Latency{u} && (\forall (u,v) \in E) \label{con:latency2} \\
        && \Latency{v} & \ge x_{vi} \cdot \Finish{i} && (\forall v) \ (\forall i=1,...,k) \label{con:latency_finish} \\
        && x_{vi} & \in \{0,1\} && (\forall v) \ (\forall i=0,...,k) \nonumber
    \end{alignat}
    \caption{A schema of the Integer Program for latency minimization}
    \label{fig:ip_for_latency}
\end{figure*}

Our IP formulation is presented in \cref{fig:ip_for_latency}.
Devices/subgraphs of accelerators are indexed $i=1,...,k$, and the special index $i=0$ denotes all CPU cores together.
We use binary variables $x_{vi}$ to denote whether node $v$ should be placed on device/subgraph $i$, and continuous variables $\Latency{v}$ to denote the time at which node $v$ has finished executing and its output (or that of the subgraph where it is placed) is available in RAM. 
The objective $\TotalLatency$ is the maximum of $\Latency{v}$ over all nodes $v$.
All variables except $x_{vi}$ are bound to be non-negative (i.e., not necessarily integral).
We explain the remaining variables and constraints below:
\begin{itemize}
    \item The variable $\Incoming{ui}$ is intended to be $1$ if $u$ is not in subgraph $i$, but has an edge to it (and $0$ otherwise). In this case its output needs to be transferred to the corresponding accelerator's memory. This is encoded by constraint~\eqref{con:incoming}.
    \item Similarly $\Outgoing{ui}$ should be $1$ if $u$ is in subgraph~$i$ and has an edge going out of $i$. In this case its output needs to be transferred from the corresponding accelerator to RAM. This is encoded by constraint~\eqref{con:outgoing}.
    \item For a subgraph $i$, $\Start{i}$ is the time at which all its inputs are ready in RAM (not the accelerator's memory). This is encoded by constraint~\eqref{con:latency_incoming}. $\Finish{i}$ is the time at which all its outputs are ready and have been transferred to RAM. Constraint~\eqref{con:starttofinish} relates the two by taking into account the in-transfer, processing inside the subgraph, and the out-transfer.
    \item Constraint~\eqref{con:assigned} means that every node should be assigned to exactly one subgraph (or a CPU).
    \item Constraint~\eqref{con:memory} encodes the requirement that the sum of sizes of nodes on accelerator $i$ should be at most $\maxfpgasize$.
    \item Constraints~\eqref{con:latency1} and \eqref{con:latency2} encode that node $v$ can start processing once all of its predecessors $u$ are finished.
    If $v$ is placed on a CPU, then its processing takes $\cpu{v}$ time.
    Otherwise, its processing time is taken into account in constraint~\eqref{con:starttofinish} of the subgraph $i$ where it is placed; the outputs of $i$ are available at time $\Finish{i}$,
    and constraint~\eqref{con:latency_finish} will set $\Latency{v}$ to that value.
\end{itemize}

Note that the formulation as presented in \cref{fig:ip_for_latency} is not yet a Mixed-Integer Program (MIP) -- but can be made so.

\begin{lemma} \label{lem:reformulate}
    The constraints \eqref{con:contiguous}, \eqref{con:latency_incoming} and \eqref{con:latency_finish} can be reformulated as linear constraints.
\end{lemma}
\begin{proof}
    To reformulate \eqref{con:latency_incoming},
    take $H$ to be a very large number (guaranteed to be larger than $\Latency{v}$ in any considered solution)
    and write
    \[ \Start{i} \ge \Latency{v} - (1 - \Incoming{vi}) \cdot H \,. \]
    If $\Incoming{vi} = 1$, then we recover the original constraint.
    Otherwise, if $\Incoming{vi} = 0$, the right-hand side is negative and the constraint becomes vacuous.
    Constraint \eqref{con:latency_finish} can be rewritten analogously.

    To formulate the contiguity constraint \eqref{con:contiguous}, we use extra variables $z_{vi}$, with the following linear constraints:
    \begin{align}
        z_{vi} &\ge x_{vi} \quad & (\forall v) \ (\forall i=1,...,k) \label{con:z} \\
        z_{vi} &\le z_{ui} \quad & (\forall (u,v) \in E) \ (\forall i=1,...,k) \label{con:z2} \\
        z_{vi} &\le x_{vi} - x_{ui} + 1 \quad & (\forall (u,v) \in E) \ (\forall i=1,...,k) \label{con:z3}
    \end{align}
    Intuitively, one can think of $z_{\cdot i}$ as being a non-increasing sequence that lays above $x_{\cdot i}$.
    
    Fix $i$.
    We claim that the subgraph $S = \{v \in V : x_{vi} = 1\}$ is contiguous if and only if there exists a vector $(z_{vi})_{v \in V}$ satisfying constraints \eqref{con:z}--\eqref{con:z3}.
    
    \textbf{"Only if" direction:} for every $v$ define $z_{vi} = 1$ if any node in $S$ is reachable from $v$, and $0$ otherwise. Constraints \eqref{con:z} and \eqref{con:z2} are clearly satisfied.  For constraint \eqref{con:z3}, the only interesting case is when $x_{vi} = 0$ and $x_{ui} = 1$; then the constraint becomes $z_{vi} \le 0$. This is indeed satisfied as no node $w \in S$ can be reachable from $v$; if it were, then the triple $(u,v,w)$ would contradict the contiguity of $S$ (cf. \cref{def:contiguous}).
    
    \textbf{"If" direction:} 
    towards a contradiction assume that there are nodes $u \in S$, $v \not \in S$ and $w \in S$ such that $v$ is reachable from $u$ and $w$ is reachable from $v$.
    Without loss of generality assume that $(u,v) \in E$.
    Then $z_{vi} \le 0$ by constraint~\eqref{con:z3}.
    By following the path from $v$ to $w$ and repeatedly applying constraint~\eqref{con:z2} we get $z_{wi} \le z_{vi}$, thus $z_{wi} \le 0$.
    But by constraint~\eqref{con:z} we must also have $z_{wi} \ge 1$ since $w \in S$, a contradiction.
\end{proof}

Our formulation has $O(|V| \cdot k)$ variables and $O((|V| + |E|) \cdot k)$ constraints.

\subsection{Non-contiguous splits}
\label{sec:ip_for_latency_noncontig}

Our formulation can be extended to allow every accelerator to hold
up to some number $\singlesubgraphcount$ of contiguous subgraphs.
We then need to ensure that their processing times in our schedule do not overlap.

We use a modified Integer Program that provides for a customizable extent of non-contiguity.
Here, an accelerator can be assigned several subsets of nodes $S \subseteq V$,
each of which we will call a \emph{subgraph}.
The mode of computation described 
at the beginning of \cref{sec:latency}
is used for every subgraph.
We require every subgraph to be
a contiguous set $S$ of nodes.

We index devices/subgraphs as follows.
For each accelerator $i=1,...,k$ we create $\singlesubgraphcount$ subgraph slots
indexed $j=(i-1) \singlesubgraphcount + 1,(i-1) \singlesubgraphcount + 2, ..., i \singlesubgraphcount$,
where $\singlesubgraphcount$ is a customizable degree of non-contiguity that can be adjusted for the workload at hand.
The special index $j=0$ will denote all CPU cores together.

The modified IP formulation is given in \cref{ip_for_latency_noncontig}.

\begin{figure*}
    \begin{alignat}{4}
        \text{min} \quad && \TotalLatency \nonumber \\
        \text{s.t.} \quad && \sum\nolimits_{j=0}^{\subgraphcount} x_{vj} \ & = 1 & \quad & (\forall v) \tag{\ref{con:assigned}} \\
        && \text{subgraph } &\{ v \in V : x_{vj} = 1 \} \text{ is contiguous} && (\forall j>0) \tag{\ref{con:contiguous}} \\
        && \maxfpgasize & \ge \sum\nolimits_v \mem{v} \cdot \sum\nolimits_{j=(i-1) \singlesubgraphcount + 1}^{i \singlesubgraphcount} x_{vj} && (\forall i=1,...,k) \tag{\ref{con:memory}*} \\
        && \Incoming{uj} & \ge x_{vj} - x_{uj} && (\forall (u,v) \in E) \ (\forall j>0) \tag{\ref{con:incoming}} \\
        && \Outgoing{uj} & \ge x_{uj} - x_{vj} && (\forall (u,v) \in E) \ (\forall j>0) \tag{\ref{con:outgoing}} \\
        && \TotalLatency & \ge \Latency{v} && (\forall v) \nonumber \\
        && \Start{j} & \ge \Latency{v} \cdot \Incoming{vj} && (\forall v) \ (\forall j>0) \tag{\ref{con:latency_incoming}} \\
        && \Finish{j} & = \Start{j} 
        + \sum\nolimits_v \Incoming{vj} \cdot \comm{v} \nonumber \\
        &&& + \sum\nolimits_v x_{vj} \cdot \fpga{v}
        + \sum\nolimits_v \Outgoing{vj} \cdot \comm{v} && (\forall j>0) \tag{\ref{con:starttofinish}} \\
        && \Latency{v} & \ge x_{v0} \cdot \cpu{v} && (\forall v) \tag{\ref{con:latency1}} \\
        && \Latency{v} & \ge x_{v0} \cdot \cpu{v} + \Latency{u} && (\forall (u,v) \in E) \tag{\ref{con:latency2}} \\
        && \Latency{v} & \ge x_{vj} \cdot \Finish{j} && (\forall v) \ (\forall j>0) \tag{\ref{con:latency_finish}} \\
        && \Start{j} & \ge \Finish{j-1} && (\forall j>0, j \ne 1 \text{ mod $\singlesubgraphcount$}) \label{con:finishtostart} \\
        && x_{vj} & \in \{0,1\} && (\forall v) \ (\forall j) \nonumber
    \end{alignat}
    \caption{A schema of the Integer Program for latency minimization  (non-contiguous splits: $\singlesubgraphcount$ contiguous subgraphs per accelerator).}
    \label{ip_for_latency_noncontig}
\end{figure*}

We discuss the constraints that differ from the contiguous version:

\begin{itemize}
    \item Constraint~(\ref{con:memory}*) encodes the requirement that the sum of sizes of nodes in all subgraphs that are placed on accelerator $i$ should be at most $\maxfpgasize$.
    \item Constraint~\eqref{con:finishtostart} arises because an accelerator $i$ cannot process more than one subgraph at a time. Therefore we order its subgraphs $j = (i-1) \singlesubgraphcount + 1, ..., i \singlesubgraphcount$ by the time when they are processed.
\end{itemize}

Finally, if collocation constraints are required (e.g.~for training), then they should be expressed in terms of devices rather than subgraphs.
That is, for two nodes $u$ and $v$ that should be collocated, we write $x_{u0} = x_{v0}$ 
and
for $i=1,...,k$,
$\sum_{j=(i-1)\singlesubgraphcount + 1}^{i \singlesubgraphcount} x_{uj} = \sum_{j=(i-1)\singlesubgraphcount + 1}^{i \singlesubgraphcount} x_{vj}$.

Our formulation has $O(|V| \cdot q \cdot k)$ variables and $O((|V| + |E|) \cdot q \cdot k)$ constraints.

\subsection{Non-pipelined model-parallel training}
The algorithm described above can be directly applied to the traditional setting of model-parallel \emph{training} with no pipelining (one sample at a time, as shown in \cref{fig:train_mp}).
In this case the computation graph contains a forward-pass part followed by a backward-pass part.
A natural extra requirement is that corresponding forward and backward nodes be placed on the same device, as they operate on the same set of weights.
It is easy to express this co-location constraint: for forward and corresponding backward nodes $u$ and $v$ we require $x_{ui} = x_{vi}$ for all $i$.
The contiguity constraint (see \cref{sec:ip_for_latency_noncontig} below) should be enforced separately for the forward and the backward parts.

\section{Throughput Maximization}
\label{sec:throughput}

The next goal of this work is to provide an algorithm for the setting where the DNN handles a steady stream of samples and the metric of interest is \emph{throughput}.
For simplicity we think that there are $n \to \infty$ samples to be processed offline.
A schedule of choice in this scenario is model parallelism with \emph{pipelining}.
Without pipelining, only one device is active at any given time (see \cref{fig:inf_mp,fig:train_mp}), which leads to under-utilization of resources.
We remark that pipelining \emph{schedules} that we discuss are essentially due to prior works~\cite{huang2019gpipe,narayanan2018pipedream},
which discuss their implementation aspects, statistical efficiency,
and demonstrate large real-world gains in time-to-accuracy.
Here we focus on \emph{algorithms} to find optimal \emph{splits} for this mode of execution.

We begin by introducing our techniques in the setting of \emph{pipelined inference}, as it is simpler yet allows us to present the main ideas.
Next, we will extend them to handle \emph{training} workloads.

\begin{figure}
\centering
\subfloat[Pipeline with contiguous splits]{
\includegraphics[width=0.5\textwidth]{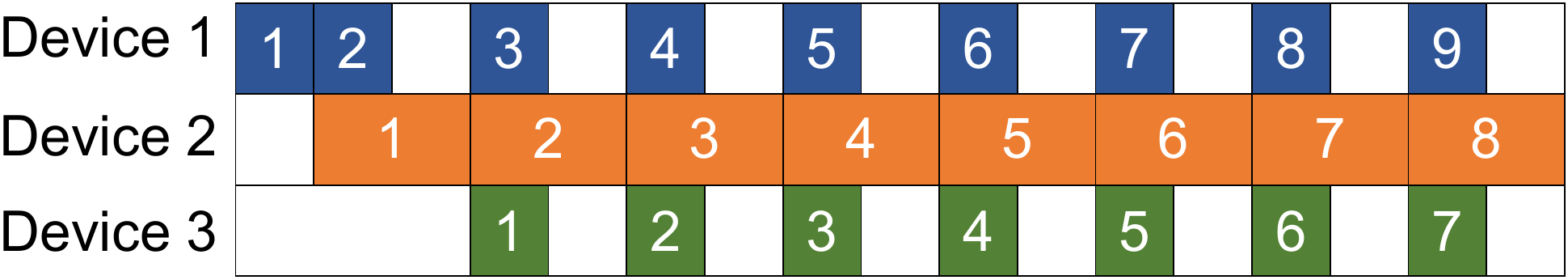}
\label{fig:inf_pipe}}
\subfloat[Pipeline with non-contiguous splits]{
\includegraphics[width=0.5\textwidth]{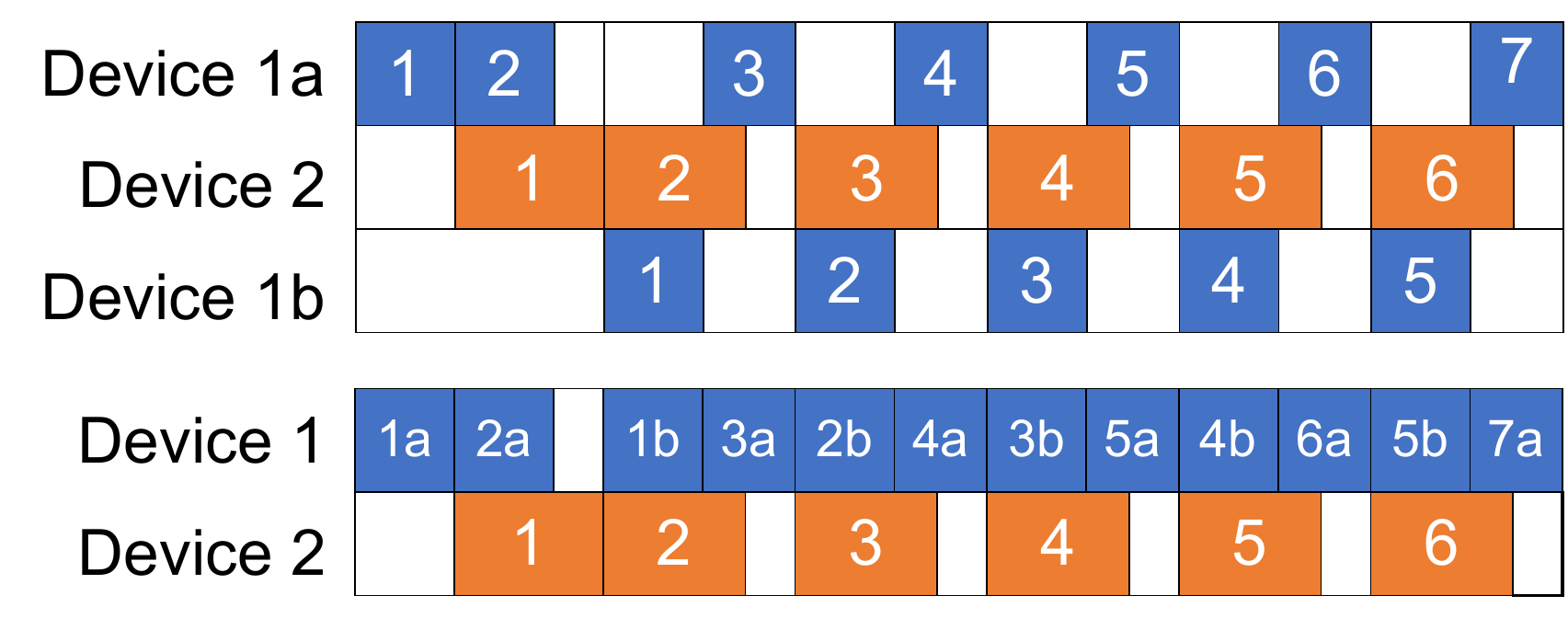}
\label{fig:inf_pipe_noncontiguous}}
\caption{Pipelined inference.  In these figures $x$-axis is time, rectangle widths are device loads (processing times for a sample), and numbers 1--9 are input sample/minibatch identifiers.  The average time spent per sample is decided by the most loaded and always-busy device. In (b) the non-contiguous subgraph of device 1 can be split into two contiguous subgraphs and thought to be assigned to virtual sub-devices 1a and 1b (that can never be executing at  same time). The top and bottom figures in (b) present two equivalent ways to view this schedule.}
\label{fig:pp}
\end{figure}

\subsection{Inference and throughput maximization}
\label{sec:inference_and_throughput_maximization}

Imagine a DNN that has already been contiguously split into subgraphs per device.
The question we ask is: How do we schedule the execution of many samples so as to maximize the throughput, or equivalently, minimize the average Time-Per-Sample?
To do so, we use pipelined inference, i.e, we build a pipeline out of devices (in the order in which the subgraphs are arranged in the DNN), and we insert consecutive samples into it (see \cref{fig:inf_pipe}).
Time can be viewed as divided into rounds: each sample spends one round on every device. After a short ramp-up period, the pipeline reaches a steady state, in which the duration of every round is determined by the slowest (most loaded) device.
With a batch of $n$ samples, the average Time-Per-Sample becomes just the maximum load of a device (plus a vanishing $O(1/n)$ term for the ramp-up and ramp-down periods).
We remark that this schedule is optimal, in the sense that that this average time cannot be lower: the bottleneck device would need to spend at least ($n$ $\times$ its load) time to process $n$ samples
in \emph{any} schedule.

The above discussion shows that
the best split is one that minimizes the maximum load of a device.
Under this \emph{min-max} objective, it is optimal to balance the load among devices,
in contrast to the \emph{min-sum}-like objective of latency minimization.
Another difference is that
here, when searching for the best split,
we do not need to simultaneously optimize for the best schedule
(which was done for latency minimization using the $\Latency{v}$ variables in \cref{sec:latency})
-- pipelining gives this for free. Without this scheduling aspect, we are left with a partitioning problem, which is easier to solve.

\subsubsection{Dynamic Programming solution}
\label{sec:dp}

The two main ideas behind our Dynamic Programming (DP) solution are described below.
First, if we want {contiguous} splits, then we can carve out successive device-subgraphs starting from the beginning of the network.
At all times, the already-partitioned region will be a downward-closed set that we henceforth call an \emph{ideal}.

\begin{definition}
    We call a set $I \subseteq V$ of nodes an \emph{ideal}
    if for any $(u,v) \in E$ with $v \in I$ we have $u \in I$.
\end{definition}

It turns out that, going from ideal to ideal, we can obtain every possible contiguous subgraph:

\begin{fact}
    \label{fact:contiguous_is_difference_of_two_ideals}
    A set $S \subseteq V$ of nodes is contiguous (see \cref{def:contiguous}) if and only if it is the difference of two ideals: $S = I \setminus I'$ where $I' \subseteq I$.
\end{fact}
\begin{proof}
    \textbf{"Only if" direction:}
    we take $I = \{ v \in V : \text{ some node in $S$ is reachable from $v$}\}$
    and
    $I' = I \setminus S$.
    Clearly $I$ is an ideal and $S = I \setminus I'$;
    it remains to show that $I'$ is an ideal.
    For this, take any edge $(u,v) \in E$ with $v \in I'$;
    we need to show that $u \in I'$.
    Since $v \in I$, we also have $u \in I$.
    It remains to show that $u \not \in S$.
    Assume otherwise, i.e., that $u \in S$.
    Since $v \in I$, some node $w \in S$ is reachable from $v$.
    But $v \not \in S$; thus the triple $(u,v,w)$ contradicts the contiguity of $S$.
    
    \textbf{"If" direction:}
    towards a contradiction assume that there are nodes $u \in S$, $v \not \in S$ and $w \in S$ such that $v$ is reachable from $u$ and $w$ is reachable from $v$.
    We have $w \in S \subseteq I$ and $I$ is an ideal, so $v \in I$.
    Since $v \not \in S = I \setminus I'$, we must have $v \in I'$.
    Since $I'$ is an ideal, also $u \in I'$.
    However, $u \in S = I \setminus I'$, a contradiction.
\end{proof}


General DAGs can contain exponentially many ideals (the worst case being a graph with no edges).
Our second insight is that the operator graphs of most modern DNNs, while less and less linear in structure, still contain a manageable amount of branching.
This topology ensures a limited number of ideals.
Thus, we can consider all possible contiguous sets via Dynamic Programming.

\newcommand{\dpstate}[3]{\mathrm{dp}[#3][#1][#2]} 

\paragraph{The Dynamic Program.} 
We fill a DP table of dimensions $(k+1) \times (\ell+1) \times (\text{number of ideals in $G$})$, where the cell $\dpstate{k'}{\ell'}{I}$ is intended to hold the optimal (i.e.~smallest) maximum load of a device if we use $k'$ accelerators and $\ell'$ CPUs to partition the set $I \subseteq V$ of nodes.
The initialization, which is $(k',\ell') = (0,0)$, is easy:
the only ideal that we can partition using $0$ devices is the empty set, so we have $\dpstate{0}{0}{I} = 0$ if $I = \emptyset$ and $\infty$ otherwise.
For $(k',\ell') \ne (0,0)$ and any $I$, we iterate over all choices of the subgraph being placed on the last device (which is either a CPU or an accelerator), which are contiguous sets of the form $I \setminus I'$ for an ideal $I' \subseteq I$:

\begin{align*}
    \small \dpstate{k'}{\ell'}{I} = 
    \hspace*{-0.7em}\min_{\text{ideal} I' \subseteq I}\hspace*{-0.7em} \min[\max\left( \dpstate{k'-1}{\ell'}{I'}, \fpganoindex(I \setminus I') \right), 
        \max\left( \dpstate{k'}{\ell'-1}{I'}, \cpunoindex(I \setminus I') \right)]
\end{align*}

with the caveat that if $k' = 0$ or $\ell' = 0$, then we should skip the corresponding branch of the second $\min$.
By $\cpunoindex(S)$ and $\fpganoindex(S)$ we denote the total load of the corresponding device holding the contiguous set $S$; thus $\cpunoindex(S) = \sum_{v \in S} \cpu{v}$, and
$\fpganoindex(S)$ comprises: the incoming communication costs of $S$ ($\sum_v \comm{v}$ over $v \not \in S$ with an edge to $S$), the processing cost $\sum_{v \in S} \fpga{v}$, and the outgoing communication costs of $S$ ($\sum_v \comm{v}$ over $v \in S$ with an edge to $V \setminus S$).
If $S$ would not fit on an accelerator, i.e., $\sum_{v \in S} \mem{v} > \maxfpgasize$, then we instead set $\fpganoindex(S) = \infty$.

\paragraph{Runtime and memory usage.}
The DP table dominates the memory usage, which is
$O(\numideals \cdot (k+1) \cdot (\ell+1))$, where by $\numideals$ we denote the number of ideals in $G$.
It takes $O(\numideals)$ time to fill one entry of the table.
The entire DP solution can be implemented to run in time $O(\numideals^2 \cdot [(k+1) \cdot (\ell+1) + |V| + |E|])$, where the additional term $\numideals^2 \cdot (|V| + |E|)$ arises due to computing the costs $\cpunoindex(I \setminus I')$ and $\fpganoindex(I \setminus I')$.

\paragraph{Extensions.}
A similar DP solution is used in PipeDream~\cite{narayanan2018pipedream},
albeit only for layer-granularity graphs that are linear (i.e.,~a path).
That work also considers two extensions:
replication (where a single subgraph is replicated on multiple devices,
creating a hybrid model-parallel/data-parallel split)
and hierarchical accelerator topologies (e.g. clusters of GPUs connected internally with faster interconnects).
Both of these extensions can also be handled by our DP algorithm,
at the costs of $O(k+\ell)$ and $O(\numideals)$ factors in the runtime, respectively.
See \cref{sec:extensions} for more details.

\subsubsection{Dynamic Programming Solution -- Linearization Heuristic (DPL)}
\label{sec:linearization}

The $\numideals^2$ term in the running time
can make the DP solution inefficient
for certain DNN workloads that are both large and strongly branching.
To deal with this,
one can reduce the search space by adding artificial edges to the graph.
In particular,
we use the following version of this technique:
find a Hamiltonian path
(in other words, a linear/topological ordering)
of the input DAG
using a Depth-First Search (DFS) traversal,
and add this path of artificial edges.
This yields
the largest possible reduction of the search space:
the resulting graph has only one topological ordering,
and thus the number $\numideals$ of ideals 
becomes the number $|V|$  of nodes plus $1$,
giving an $O(|V|^2)$ term instead of $O(\numideals^2)$.
The algorithm so obtained is polynomial-time,
but it may not return the optimal solution.
In Section~\ref{sec:experiments} we show that it is very close to optimal for most workloads and provides a compelling trade-off between solution quality and runtime. We denote it by DPL in that section.

\subsubsection{Integer Programming solution}
\label{sec:inference_ip}

\begin{figure*}[t]
    \begin{alignat}{4}
        \text{min} \quad && \MaxLoad \nonumber \\
        \text{s.t.} \quad && \sum\nolimits_{i=1}^{k+\ell} x_{vi} \ & = 1 & \quad & (\forall v) \label{con:ml_assigned} \\
        && \text{the subgraph } \{ v \in V : x_{vi} = 1 \} & \text{ is contiguous \textbf{(optional)}} && (\forall i)
        \label{con:ml_contiguous} \\
        && \maxfpgasize & \ge \sum\nolimits_v \mem{v} \cdot x_{vi} && (\forall i=1,...,k) \label{con:ml_memory} \\
        && \Incoming{ui} & \ge x_{vi} - x_{ui} && (\forall (u,v) \in E) \ (\forall i=1,...,k) \label{con:ml_incoming} \\
        && \Outgoing{ui} & \ge x_{ui} - x_{vi} && (\forall (u,v) \in E) \ (\forall i=1,...,k) \label{con:ml_outgoing} \\
        && \MaxLoad & \ge \Load{i} && (\forall i) \nonumber \\
        && \Load{i} = \sum\nolimits_v \Incoming{vi} \cdot \comm{v} & + \sum\nolimits_v x_{vi} \cdot \fpga{v}
        + \sum\nolimits_v \Outgoing{vi} \cdot \comm{v} && (\forall i=1,...,k) \label{con:ml_load_fpga} \\
        && \Load{i} & = \sum\nolimits_v x_{vi} \cdot \cpu{v} && (\forall i=k+1,...,k+\ell) \label{con:ml_load_cpu} \\
        && x_{vi} & \in \{0,1\} && (\forall v) \ (\forall i) \nonumber
    \end{alignat}
    \caption{A schema of the Integer Program for max-load minimization (throughput maximization). See \cref{lem:reformulate} on how to reformulate constraint~\eqref{con:ml_contiguous} to obtain an Integer Program.}
    \label{fig:ip_for_maxload}
\end{figure*}

Our IP formulation is presented in \cref{fig:ip_for_maxload}.
We index devices as follows: accelerators are assigned indices $i=1,...,k$
and CPUs are indexed $i=k+1,...,k+\ell$.
As in \cref{sec:latency}, we use binary variables $x_{vi}$ to denote whether node~$v$ should be placed on device~$i$.
The variables $\Incoming{vi}$, $\Outgoing{vi}$ and constraints \eqref{con:ml_assigned}--\eqref{con:ml_outgoing} are also analogous to those used in the latency-minimization IP.
However, this IP is simpler as, thanks to the maximum-load objective, no scheduling aspect is present.
The objective $\MaxLoad$ is the maximum over $\Load{i}$ for all devices $i$,
which is given by constraint~\eqref{con:ml_load_fpga} for accelerators
and \eqref{con:ml_load_cpu} for CPUs.

\subsection{Non-contiguous splits}
\label{sec:non-contiguous-splits}

Suppose we are given a non-contiguous split of a DNN.
We go back to the question from \cref{sec:inference_and_throughput_maximization}:
how to best schedule our workload?
Clearly, we still cannot obtain a smaller average time per sample than
the max-load.\footnote{However, the optimal max-load of a non-contiguous split can be lower than the best contiguous one.}
Fortunately, we can still match the max-load using a variant of pipelining.
A challenge here is that the device-subgraphs may no longer have a linear or acyclic ordering induced from the input DNN (e.g. in \cref{fig:ncontig_split}, neither subgraph comes fully before the other).
One possible solution (see \cref{fig:inf_pipe_noncontiguous} for an example) is to split non-contiguous subgraphs into smaller ones, so that all subgraphs can be topologically ordered, and mentally place them on \emph{virtual devices}; then we build a pipeline of virtual devices.
Now we can build a round-based schedule as before, keeping in mind that virtual devices belonging to the same real device cannot process concurrently.
The bottleneck device will be the one whose total load (of all virtual devices) is maximal.\footnote{This quantity is the original load of that device, independent of the split into virtual devices.}
See \cref{fig:inf_pipe_noncontiguous} for an example.

The above discussion shows that our max-load objective does not change when dealing with non-contiguous splits. 
Our IP solution (\cref{sec:inference_ip})
"natively supports" the non-contiguous setting,
by just removing
the contiguity constraint.

\subsection{Training and throughput maximization}
\label{sec:training_and_throughput_maximization}

\newcommand{\FW}[1]{\mathrm{FW}_{#1}}
\newcommand{\BW}[1]{\mathrm{BW}_{#1}}

Pipeline parallelism can be applied to training as well, where the task of processing large numbers of samples (and maximizing throughput) is especially relevant.
As discussed at the end of \cref{sec:latency}, computation graphs for training consist of a forward-pass part and a backward-pass part.
Certain backward nodes operate on the same state (weights/parameters) as their corresponding forward nodes, and
so they must be colocated.
Contiguity, if desired, should be enforced separately for the forward and the backward parts; i.e., a device $i$ would hold a contiguous subgraph of the forward part and a contiguous subgraph of the backward part.
Let $\FW{i}$ and $\BW{i}$ denote their respective loads/costs.

\begin{figure}[t!]
\centering
\subfloat[GPipe schedule (batch with $4$ microbatches).]{
\includegraphics[width=0.5\textwidth]{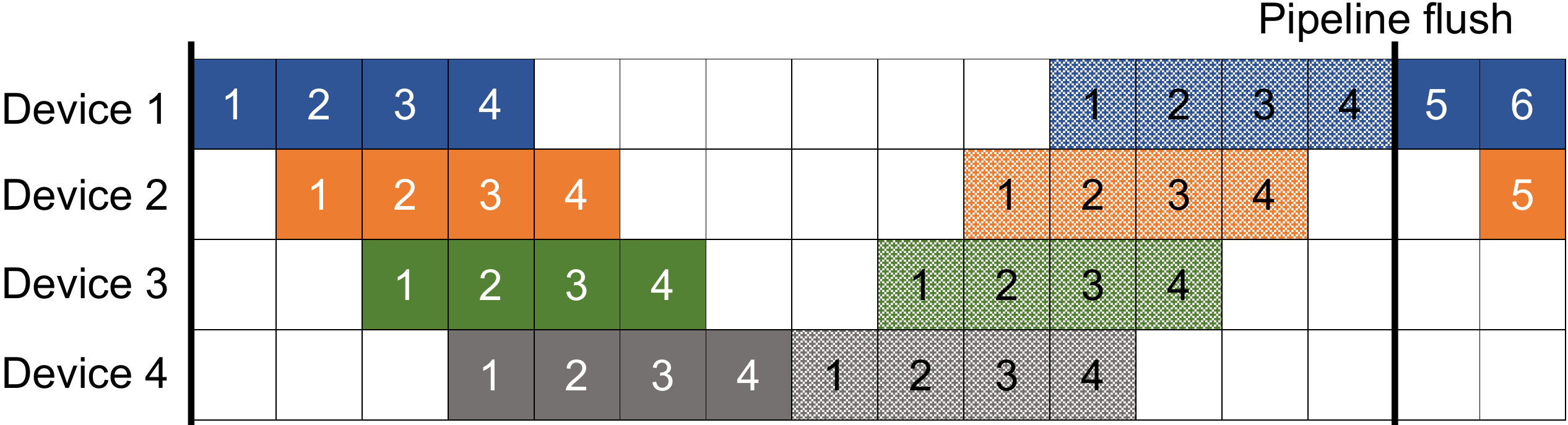}
\label{fig:train_gpipe}}
\subfloat[PipeDream's 1F1B schedule.]{
\includegraphics[width=0.5\textwidth]{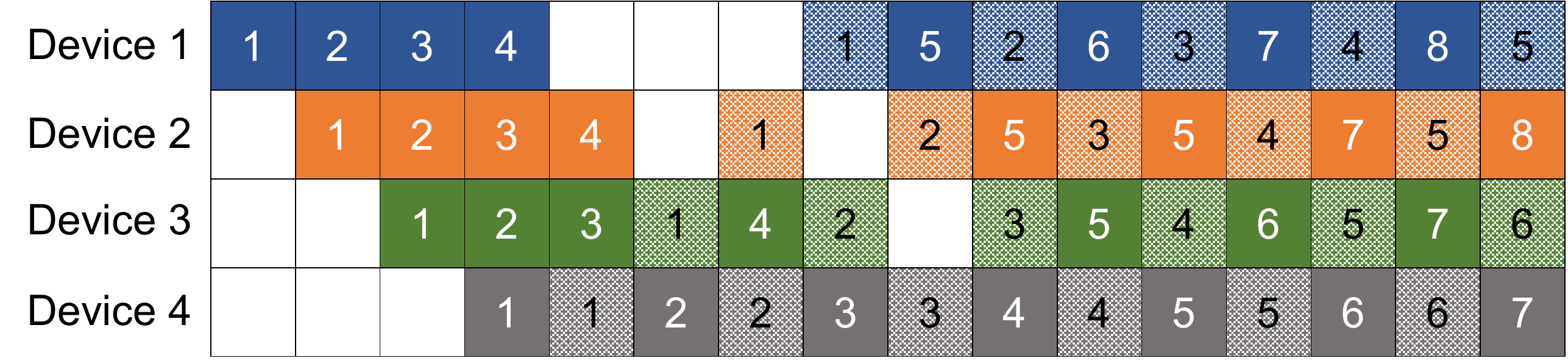}
\label{fig:train_pd}}\\
\caption{Pipeline-parallel training schedules. For simplicity the load is drawn as equal for all devices and for the forward (darker color) and backward passes (lighter color).}
\label{fig:training_schedule}
\end{figure}

\paragraph{Objective.}
GPipe~\cite{huang2019gpipe} and PipeDream~\cite{narayanan2018pipedream}
propose two different pipeline schedules.
Our max-load objective function is appropriate for both these schedules.
In \textbf{PipeDream} -- see \cref{fig:train_pd} -- after a short ramp-up period, every device starts alternating between processing a forward sample and a backward sample, which together takes $\FW{i} + \BW{i}$ time.
As before, the device $i$ that maximizes this quantity, i.e., the load,
is the bottleneck that decides the throughput of the system.
The \textbf{GPipe} schedule, shown in \cref{fig:train_gpipe}, first processes all forward samples in a batch, pipelined as they would be for inference; the average time taken for a sample in this pass is $\max_i \FW{i}$ (ignoring an $O(1/n)$ term).
The backward pass then takes place, with an average time of $\max_i \BW{i}$.
We thus get the objective $\max_i \FW{i} + \max_{i} \BW{i}$; the difference between this and the objective $\max_i \FW{i} + \BW{i}$ is insignificant, as we argue in \cref{sec:max_fw_bw}. 
For \emph{non-contiguous splits}, both types of schedules can be modified
in the same vein as in \cref{sec:non-contiguous-splits}.

Next we describe how to extend our algorithms from \cref{sec:inference_and_throughput_maximization}
for training workloads.

\paragraph{Integer Programming:} 
Our IP solution handles training graphs out-of-the-box; the only required modification is that we apply the contiguity constraint~\eqref{con:ml_contiguous} separately for the forward and the backward parts (if desired).

\paragraph{Dynamic Programming:}
Our DP algorithm can only find contiguous splits, but now most devices need to be assigned two contiguous subgraphs (backward and forward).
Our solution is to run the DP only on the forward part, but taking the corresponding backward nodes together with every considered contiguous subgraph (we also count their cost). 
Some care is required to make sure that we assign those backward nodes that do not have a corresponding forward node, and that we only consider those contiguous subgraphs on the forward side whose corresponding backward nodes also form contiguous subgraphs.
See \cref{app:preprocessing} for more details.
%

\section{Experiments -- Throughput Maximization}
\label{sec:experiments}

In this section we evaluate our throughput maximization algorithms on the following modern DNN models for inference and training: BERT (with 3, 6, 12, and 24 Transformer layers), ResNet50, Inception-v3, and GNMT.
The results for latency minimization can be found in \cref{sec:latency_experiments}.

We reiterate that we do not evaluate a particular pipelining {system},
but \emph{algorithms} to find high-quality splits for pipelined executions.
However, we remark that our max-load objective function (cost model) is a natural generalization of that of PipeDream, which has been shown~\cite[Figure 15]{narayanan2018pipedream} to closely reflect real performance.

The code and workloads used for evaluations are available at\\
\url{https://github.com/msr-fiddle/dnn-partitioning}.

\paragraph{Inputs (workloads).}
We evaluate our algorithms on diverse and widely used deep learning workloads ranging from transformer models (BERT) and convolutional neural networks (ResNet, Inception) to translation LSTM-based models (GNMT).
We exported BERT~\cite{bert} and ResNet~\cite{resnet} operator graphs through the ONNX Runtime library~\cite{onnxruntime}.
It allows exporting the operator graph topology for deep learning models by taking as input their forward pass and appending the corresponding backward pass to generate an output in ONNX format. 
We obtained all the layer graphs from previous work~\cite{harlap2018pipedream}.

Our inputs correspond to the following deployment scenarios for these workloads.
The DNN workloads are split across 6 accelerators of the same type (GPU for layer graphs, a hardware accelerator representing TPUs or FPGAs for operator graphs).
We use 3 accelerators in case of the smaller BERT-3 and BERT-6 models.
Each accelerator has 16~GB of DRAM and is connected to the CPU over a PCIE 3.0 interconnect. 
To assign a cost to each node and edge in the graph, we profile the workloads on GPU for layer workloads, and we estimate the numbers for the operator graphs for the hardware accelerator. 
We then convert the topology of each graph to a JSON format, comprising all the relevant information about the graph that is required of an input instance of our algorithms (see \cref{sec:model}).
For our Dynamic Programming solution,
we run several preprocessing steps
before we can apply our basic method;
see \cref{app:preprocessing} for more details.

\paragraph{Algorithm execution setup.}
All our experiments are executed
(that is, the optimization algorithms that we develop are run)
on a machine with an Intel Xeon E5-2673 v4 CPU and 64 GB of RAM running Ubuntu 18.04. The Dynamic Programming solution is implemented in C++ and compiled with gcc 7.4 using the \texttt{-O3} optimization flag; it is a sequential (single-threaded) implementation.
The Integer Programming formulations are solved using Gurobi 8.1~\cite{gurobi}, which runs on 4 CPU cores.
The IP models are constructed using Gurobi bindings for Python; the runtime of this construction is insignificant.

\paragraph{Baselines used for comparison.}
We use the following baselines to compare our solutions:
\begin{itemize}
\item \textbf{Hand-crafted placements}, similar to~\cite{hierarchical-2018,spotlight,placeto}.  This is still a widely used means for device placement. 
We perform expert splits only for layer graphs, as the operator graphs with their much stronger branching are infeasible to split manually.
In line with prior work~\cite{sutskever2014sequence,wu2016google}, for GNMT we place each LSTM layer on a separate GPU, and then balance between 6 devices.
We proceed similarly with BERT-24.
In ResNet50 and Inception-v3, we split the convolution, batch normalization, and ReLu layers equally among all devices.

\item \textbf{Scotch}~\cite{scotch}, a graph partitioning software used for mapping computation graphs onto devices in a balanced way, taking communication costs between dependent nodes into account.
The output splits are not guaranteed to be contiguous.

\item \textbf{Local search}~\cite{localsearch} is a heuristic that starts from a random split and repeatedly makes the best single-node reassignment until a local optimum is reached. We restart 10 times and take the best solution. Note that this almost always yields a non-contiguous split.

\item \textbf{PipeDream}~\cite{narayanan2018pipedream}'s optimizer only supports layer graphs, thus we only run it on our layer workloads. It requires the input to be a linear path, thus it contracts all branchings to single nodes.
%
\end{itemize}

\newcommand{\rot}[1]{\multicolumn{1}{c}{\makebox[2em][l]{\rotatebox{20}{#1}}}}%

\begin{table*}[t!]
\small
  \centering
  \begin{adjustbox}{width=\textwidth,center}
  \renewcommand{\arraystretch}{1.1} 
  \addtolength{\tabcolsep}{-3pt} 
  \begin{tabular}{p{2.0cm}|c|c|c|c|c|c|c|c|c|c|c|c|c|c}
    \multirow{2}{1cm}{\textbf{Workload}} & \multirow{2}{1cm}{\textbf{Nodes}} & \multicolumn{3}{c|}{\textbf{DP (contiguous)}} & \multicolumn{2}{c|}{\textbf{IP (contiguous)}} & \multicolumn{3}{c|}{\textbf{IP (non-contiguous)}} & \textbf{DPL} & \rot{\textbf{Expert}} & \rot{\textbf{Local search}} & \rot{\textbf{PipeDream}} & \rot{\textbf{Scotch}}\\
    \cline{3-15} & &  {\textbf{Ideals}} & {\textbf{Runtime}} & {\textbf{TPS}} & {\textbf{Runtime}} & {\textbf{TPS}} & {\textbf{Runtime}} & {\textbf{TPS}} & {\textbf{Gain}}& {\textbf{TPS}}& {\textbf{TPS}}& {\textbf{TPS}}& {\textbf{TPS}}& {\textbf{TPS}} \\
    \hline
    \multicolumn{7}{l}{Operator-granularity graphs, pipelined inference}\\
    \hline
    \textbf{BERT-3} & 235 & 1428 & \textbf{1s} & 27.92 &  \textbf{1s} & 27.92 & 2s & 21.91 & 27\% & 27.92 & - & 24.32 & - & 35.94 \\
    \textbf{BERT-6} & 418 & 1923 & 5s & 29.58 & \textbf{4s} & 29.58 & 54s (3s*) & 28.33 & 4\% & 29.58 & -  & 42.52 & - & 49.80 \\
    \textbf{BERT-12} & 783 & 2906 & \textbf{19s} & 147.48 & 11m (1m*) & 147.48 & $>$20m (18s*) & 130.03 & 13\% & 147.48 & - & 257.38 & - & 230.12 \\
    \textbf{ResNet50} & 604 & 241 & \textbf{0s} & 124.35 & 15s & 124.35 & 1m (10s*) & 124.35 & 0\% & 124.35 & - & 250.08 & - & 197.84 \\
    \hline
    \multicolumn{7}{l}{Operator-granularity graphs, pipelined training}\\
    \hline
    \textbf{BERT-3} & 600 & 2774 & 8s & 65.30 & 2s & 65.30 & \textbf{1s} & 54.21 & 20\% & 65.30 & - & 66.17 & - & 416.97 \\
    \textbf{BERT-6} & 1071 & 3776 & 25s & 72.86 & \textbf{6s} & 72.86 & 13m (2s*) & 71.64 & 1\% & 79.50 & - & 94.86 & - & 130.20 \\
    \textbf{BERT-12} & 2012 & 2938 & \textbf{1m} & 438.00 & $>$20m (1m*) & 438.00 & $>$20m (1m*) & 373.42 & 17\% & 438.00 & - & 737.99 & - & 800.79 \\
    \textbf{ResNet50} & 1243 & 258 & \textbf{2s} & 255.19 & 2m (28s*) & 255.19 & 7s & 255.19 & 0\% & 255.19 & - & 530.95 & - & 379.21 \\
    \hline
    \multicolumn{7}{l}{Layer-granularity graphs, pipelined inference}\\
    \hline
    \textbf{BERT-24} & 32 & 30 & \textbf{0s} & 17.79 & 1s & 17.79 & $>$20m (1s*) & 17.71 & 0.4\% & 17.79 & 20.08 & 17.80 & 17.79 & 18.03 \\
    \textbf{ResNet50} & 177 & 242 & \textbf{0s} & 33.77 & 48s & 33.77 & 14s & 33.31 & 1.3\% & 33.77 & 43.92 & 35.63 & 39.38 & 34.50 \\
    \textbf{InceptionV3} & 326 & 36596 & 32m & 51.55 & 3m & 51.55 & \textbf{19s} & 51.52 & 0\% & 51.55 & 102.48 & 54.03 & 60.42 & 54.01 \\
    \textbf{GNMT} & 96 & 17914 & 29s & 32.91 & \textbf{4s} & 32.91 & 9s & 31.68 & 4\% & 32.91 & 46.21 & 31.75 & 33.03 & 34.92 \\
    \hline
    \multicolumn{7}{l}{Layer-granularity graphs, pipelined training}\\
    \hline
    \textbf{BERT-24} & 64 & 30 & \textbf{0s} & 41.75 & 1s & 41.75 & 9s & 39.79 & 5\% & 41.75 & 49.40 & 39.93 & 41.75 & 42.01 \\
    \textbf{ResNet50} & 354 & 242 & \textbf{1s} & 78.63 & 45s & 78.63 & 15s & 76.65 & 3\% & 78.65 & 112.11 & 81.32 & 83.67 & 80.10 \\
    \textbf{InceptionV3} & 652 & 36596 & 58m & 122.76 & 8m & 123.35 & \textbf{43s} & 117.72 & 5\% & 123.93 & 213.65 & 122.80 & 128.32 & 128.32 \\
    \textbf{GNMT} & 192 & 17914 & 42s & 107.00 & 4s & 107.00 & \textbf{1s} & 88.47 & 21\% & 107.00 & 137.15 & 91.52 & 107.35 & 107.00 \\
  \end{tabular}
  \end{adjustbox}
    \caption{Pipelined workloads for maximization of throughput / minimization of Time-Per-Sample (TPS, equal to max-load). We run the IP optimizer until it guarantees a solution within 1\% of the optimum, but no longer than 20 minutes. The parenthesized times with asterisks denote the time it took the optimizer to find the solution of the final value (though it could not yet guarantee its near-optimality).
    DPL stands for the DP with the Linearization heuristic (see Section~\ref{sec:linearization}), which always runs under 3 seconds.
    The fastest non-DPL runtime for every input is in bold.
  }
  \label{tab:tput-workloads}
\end{table*}

\subsection{Results}
\cref{tab:tput-workloads} shows each workload, the number of nodes (operators or layers) in the graph, runtimes of our algorithms, and the average Time-Per-Sample (TPS) -- that is, the maximum device load, which is inversely proportional to throughput -- of the found splits. 
We also report the gain of best non-contiguous splits over best contiguous splits, and the TPS of the baselines.
For better understanding of the DP runtimes we show the number of ideals in the forward part of each DNN.
We also present the same results
    in an equivalent form,
    displaying the throughput advantages
    obtained by our algorithms
    with respect to baselines.
    See \cref{tab:tput-improvement}
    on page~\pageref{tab:tput-improvement}
    and
    \cref{fig:tput-improvement}
    on page~\pageref{fig:tput-improvement}.
Also, see \cref{fig:contig_noncont_ip_bert3_inf} for an illustration of an example pair of
optimal contiguous (top) and non-contiguous (bottom) splits of an operator-level graph (BERT-3) that are returned by our algorithms.

\begin{table}[t!]
\small
  \centering
  \begin{adjustbox}{width=17cm,center}
  \renewcommand{\arraystretch}{1.1} 
  \addtolength{\tabcolsep}{-3pt} 
\begin{tabular}{p{1.9cm}cccccccc}
    \textbf{Workload} & \textbf{DP} & \textbf{IP (contiguous)} & \textbf{IP (non-contiguous)} & \textbf{DPL} & \textbf{Expert} & \textbf{Local search}  & \textbf{PipeDream} & \textbf{Scotch} \\
    \hline
    \multicolumn{9}{l}{Operator-granularity graphs, pipelined inference}\\
    \hline
\textbf{BERT-3}      & 1.00$\times$ & 1.00$\times$ & 1.27$\times$ & 1.00$\times$ &  -  & 1.15$\times$ &  -     & 0.78$\times$\\
\textbf{BERT-6}      & 1.00$\times$ & 1.00$\times$ & 1.04$\times$ & 1.00$\times$ &  -  & 0.70$\times$ &   -    & 0.59$\times$\\
\textbf{BERT-12}     & 1.00$\times$ & 1.00$\times$ & 1.13$\times$ & 1.00$\times$ &  -  & 0.57$\times$ &   -    & 0.64$\times$\\
\textbf{ResNet50}    & 1.00$\times$ & 1.00$\times$ & 1.00$\times$ & 1.00$\times$ &  -  & 0.50$\times$ &   -    & 0.63$\times$\\
    \hline
    \multicolumn{9}{l}{Operator-granularity graphs, pipelined training}\\
    \hline
\textbf{BERT-3}      & 1.00$\times$ & 1.00$\times$ & 1.20$\times$ & 1.00$\times$ &   -    & 0.99$\times$ &   -    & 0.16$\times$ \\
\textbf{BERT-6}      & 1.00$\times$ & 1.00$\times$ & 1.02$\times$ & 0.92$\times$ &   -    & 0.77$\times$ &    -   & 0.56$\times$ \\
\textbf{BERT-12}     & 1.00$\times$ & 1.00$\times$ & 1.17$\times$ & 1.00$\times$ &   -    & 0.59$\times$ &   -    & 0.55$\times$ \\
\textbf{ResNet50}    & 1.00$\times$ & 1.00$\times$ & 1.00$\times$ & 1.00$\times$ &   -    & 0.48$\times$ &   -    & 0.67$\times$ \\
    \hline
    \multicolumn{9}{l}{Layer-granularity graphs, pipelined inference}\\
    \hline
\textbf{BERT-24}     & 1.00$\times$ & 1.00$\times$ & 1.00$\times$ & 1.00$\times$ & 0.89$\times$ & 1.00$\times$ & 1.00$\times$ & 0.99$\times$ \\
\textbf{ResNet50}    & 1.00$\times$ & 1.00$\times$ & 1.01$\times$ & 1.00$\times$ & 0.77$\times$ & 0.95$\times$ & 0.86$\times$ & 0.98$\times$ \\
\textbf{InceptionV3} & 1.00$\times$ & 1.00$\times$ & 1.00$\times$ & 1.00$\times$ & 0.50$\times$ & 0.95$\times$ & 0.85$\times$ & 0.95$\times$ \\
\textbf{GNMT}        & 1.00$\times$ & 1.00$\times$ & 1.04$\times$ & 1.00$\times$ & 0.71$\times$ & 1.04$\times$ & 1.00$\times$ & 0.94$\times$ \\
    \hline
    \multicolumn{9}{l}{Layer-granularity graphs, pipelined training}\\
    \hline
\textbf{BERT-24}     & 1.00$\times$ & 1.00$\times$ & 1.05$\times$ & 1.00$\times$ & 0.85$\times$ & 1.05$\times$ & 1.00$\times$ & 0.99$\times$ \\
\textbf{ResNet50}    & 1.00$\times$ & 1.00$\times$ & 1.03$\times$ & 1.00$\times$ & 0.70$\times$ & 0.97$\times$ & 0.94$\times$ & 0.98$\times$ \\
\textbf{InceptionV3} & 1.00$\times$ & 1.00$\times$ & 1.04$\times$ & 0.99$\times$ & 0.57$\times$ & 1.00$\times$ & 0.96$\times$ & 0.96$\times$ \\
\textbf{GNMT}        & 1.00$\times$ & 1.00$\times$ & 1.21$\times$ & 1.00$\times$ & 0.78$\times$ & 1.17$\times$ & 1.00$\times$ & 1.00$\times$ \\
  \end{tabular}
  \end{adjustbox}
  \vspace{1em}
  \caption{
    Throughput maximization results, same as in Table~\ref{tab:tput-workloads} in Section~\ref{sec:experiments},
    but presented in terms of throughput improvement
    in relation to the DP (Dynamic Program, contiguous splits) being $1\times$.
    For example, on BERT-3 inference operator-graph,
    the best non-contiguous split offers $1.27\times$ the throughput of the best contiguous one, and Scotch gives $0.78 \times$  the throughput of the best contiguous one.
    In addition, we show the single-accelerator throughput (placing the entire DNN workload on one accelerator).
    See Figure~\ref{fig:tput-improvement} for a graphical representation of data in this table.
    }
    \label{tab:tput-improvement}
\end{table}

\begin{figure}
\centering
\begin{tabular}{cc}
\subfloat[Operator graphs, inference]{\includegraphics[width=0.99\textwidth]{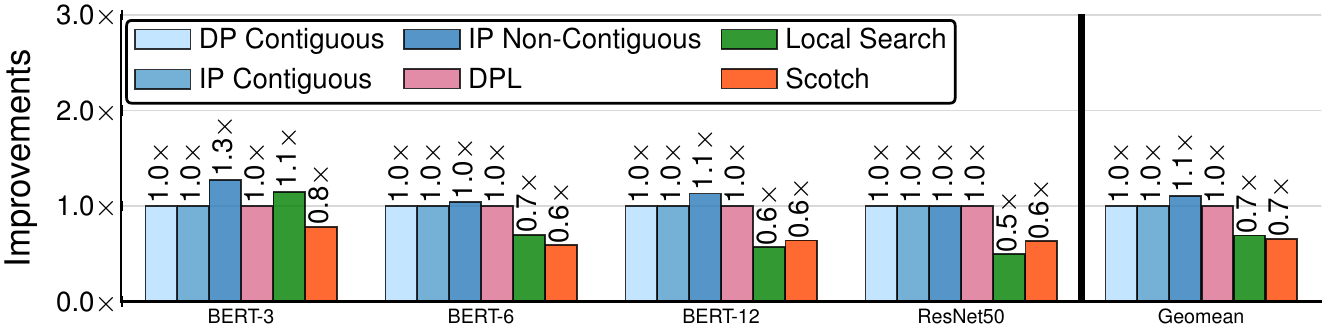}} \\
\subfloat[Operator graphs, training]{\includegraphics[width=0.99\textwidth]{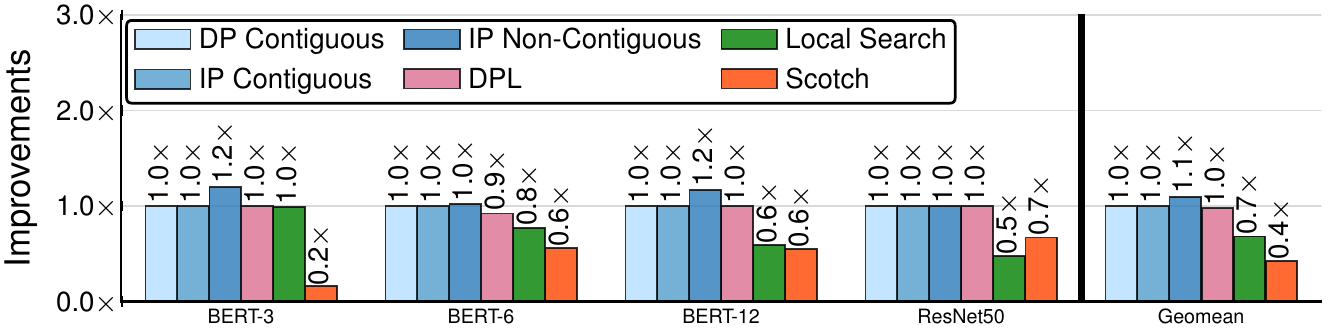}} \\
\subfloat[Layer graphs, inference]{\includegraphics[width=0.99\textwidth]{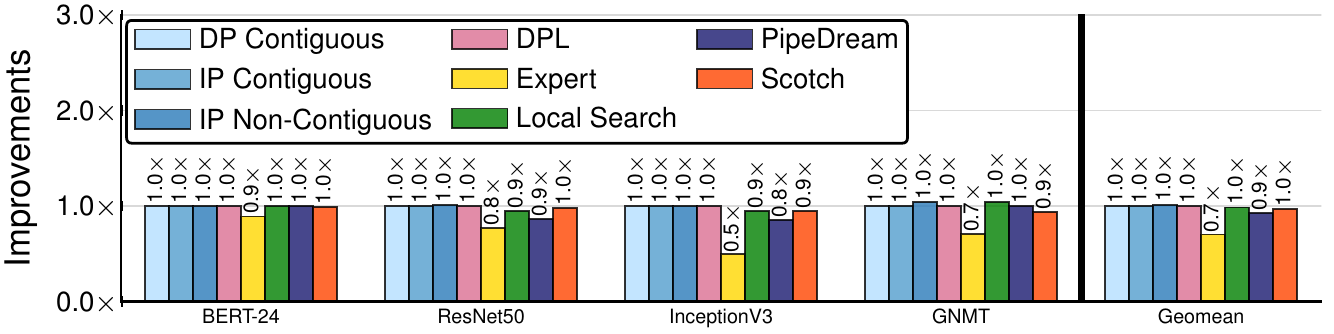}} \\
\subfloat[Layer graphs, training]{\includegraphics[width=0.99\textwidth]{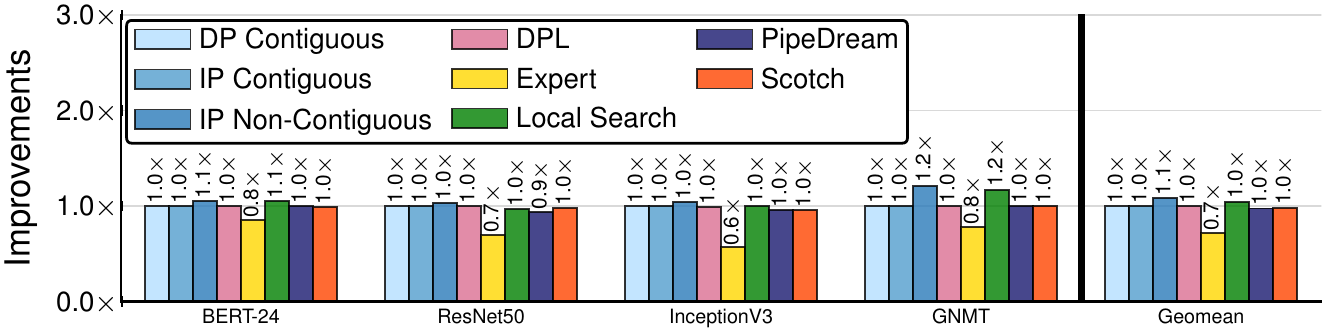}}\\
\end{tabular}
\caption{An illustration of throughput maximization results from \cref{tab:tput-improvement}, with DP (contiguous) serving as $1\times$. The blue bars are algorithms from this work, whereas the non-blue-colored bars show baselines. Plots (a) and (b) represent throughput improvements for operator-level graphs, and (c) and (d) for layer-level graphs.}
\label{fig:tput-improvement}
\end{figure}

\begin{figure*}[t!]
\centering
\includegraphics[width=\textwidth]{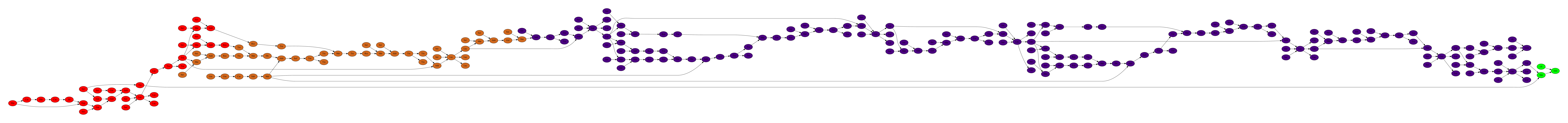}
\\
\includegraphics[width=\textwidth]{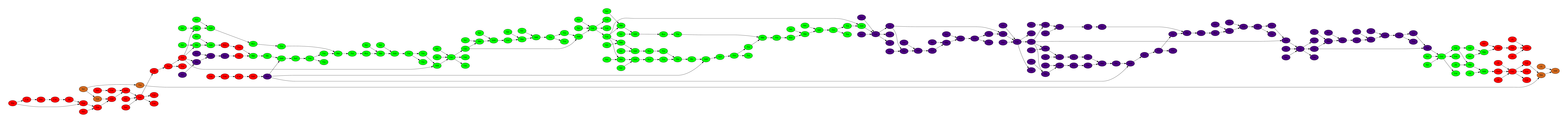}
\caption{Optimal contiguous (top) and non-contiguous (bottom) splits of the BERT-3 operator-level inference graph onto 3 accelerators and 1 CPU (for throughput maximization). Each node is colored based on its placement -- red color indicates CPU placement, and each remaining color indicates a different accelerator. The non-contiguous split achieves a 27\% higher throughput. If viewed on a computer, the figures can be zoomed in to an arbitrary degree for better inspection.}
\label{fig:contig_noncont_ip_bert3_inf}
\end{figure*}

Below we discuss the main takeaways.

\paragraph{DP vs. IP (optimality, efficiency).}
DP and IP (contiguous) both return the optimal split, so their TPS/max-load values are equal (up to a 1\% IP optimality threshold). 
The optimization problems we solve are computationally hard, and our algorithms are exponential-time in general; however, we see that their runtimes are reasonable on real-life DNN inputs due to their workload structure.
For our IP solution we used a commercial-grade solver~\cite{gurobi} that ran on 4 CPU cores.
The efficiency of free software solvers is much worse.
It is worth noting that most of the runtime is often spent on certifying the near-optimality of the found solution; it would therefore be reasonable to cut the computation much sooner, still obtaining high-quality solutions.
For the DP solution we created a single-core, self-contained implementation.
Its runtimes are very competitive with those of the IP, except for the most branching models such as Inception-v3.
We remark that in practice the DP runtime is dominated by the $O(\numideals^2 |E|)$ term and does not depend much on the numbers $k$, $\ell$ of devices unless these are very large;
in contrast,
increasing the number of accelerators can have a large impact on the IP runtime (e.g.~having 6 rather than 3 accelerators for the BERT-3 inference workload causes the IP runtime to jump from 1s to 27s).
%
Moreover, the DP runtime does not depend on the node weights, whereas the IP runtime does.

\paragraph{DPL (DP with the Linearization heuristic)} (see Section~\ref{sec:linearization}).
The DPL solution runs in time essentially $O(|V|^2|E|)$, which for our workloads is at most seconds. Crucially, the restricted search space results in a throughput loss of $9\%$ for the BERT-6 training workload, $1\%$ loss for InceptionV3 training, and no loss for all other workloads.
Therefore, DPL would be our method of choice for very large graphs; it would be able to process graphs with tens of thousands of nodes within, say, an hour.

\paragraph{Contiguous vs.~non-contiguous splits.}
Our IP solution is able to find optimal non-contiguous splits.
To the best of our knowledge, our work is the first one to examine \emph{non-contiguous splits} for pipelined model parallelism;
thus we use our experiments to evaluate the potential gains in throughput.
We observe that on average, the best non-contiguous splits offer an $\scriptstyle\sim$$10\%$ gain over the best contiguous splits; for BERT-3, the gain is as large as 20-27\%.

\paragraph{Comparison to other baselines.}
\textit{As seen in \cref{tab:tput-workloads}, our non-contiguous splits outperform all the techniques, with an improvement of up to \ncmaxexpert over \textbf{hand-crafted expert} splits (average \ncavgexpert), \ncmaxlocal over \textbf{local search} (average \ncavglocal), \ncmaxpd over \textbf{PipeDream} (average \ncavgpd), \ncmaxscotch over \textbf{Scotch} (average \ncavgscotch).}   
Hand-crafted expert placements for the layer-based graphs provide \expertc and \expertnc of the throughput in comparison to contiguous and non-contiguous splits,  respectively.
At the layer granularity, some workloads have a repetitive graph structure, which can be split manually across devices, yet this turns out to be not enough to obtain optimality.
%
Furthermore, performing a reasonable human split over operator graphs is infeasible due to the large branching and number of nodes.
%
%
Local search fares badly, underscoring the difficult, non-local structure of the optimization problem, which is also resistant to the heuristics used by Scotch.
%
%
Finally, PipeDream only considers \emph{linear} layer graphs and contracts branchings in the input graph; whereas our technique that does not contract branches is able to explore a larger search space for operator placement and achieve up to \ncmaxpd higher throughput.
%

\subsection{Advantage of operator vs. layer graphs}
\label{sec:operator_vs_layer}

\begin{table*}[t!]
\small
  \centering
  \renewcommand{\arraystretch}{1.1} 
  \addtolength{\tabcolsep}{-3pt} 
  \begin{tabular}{p{1.6cm}|P{4.2cm}|P{4.2cm}|c}
    \textbf{Workload} & \textbf{DP (run on original \newline operator graph)} & \textbf{DP (run on contracted \newline layer graph)} & \textbf{Gain} \\
    \hline
    \multicolumn{4}{l}{Operator-granularity graphs, pipelined inference}\\
    \hline
    \textbf{BERT-3} & 27.92 & 27.92 & 0\% \\ 
    \textbf{BERT-6} & 29.58 & 29.58 & 0\% \\ 
    \textbf{BERT-12} & 147.48 & 159.43 & 8\% \\ 
    \textbf{ResNet50} & 124.35 & 129.15 & 4\% \\ 
    \hline
    \multicolumn{4}{l}{Operator-granularity graphs, pipelined training}\\
    \hline
    \textbf{BERT-3} & 65.30 & 65.30 & 0\% \\ 
    \textbf{BERT-6} & 72.86 & 72.86 & 0\% \\ 
    \textbf{BERT-12} & 438.00 & 465.41 & 6\% \\ 
    \textbf{ResNet50} & 255.19 & 269.63 & 6\% \\ 
  \end{tabular}
    \caption{Throughput maximization; throughput advantage of optimization on the operator-granularity level vs.~the layer-granularity level (see Section~\ref{sec:operator_vs_layer}), for optimal contiguous splits. The numbers shown are TPS (time-per-sample).}
  \label{tab:operator_vs_layer}
\end{table*}

In this section
we measure the throughput advantage
    that can be obtained by using finer-granularity
    operator graphs
    in lieu of layer graphs.
    No conclusions on this matter can be drawn from the experimental results of \cref{sec:experiments} or \cref{sec:latency_experiments} alone,
    as our operator-graph and layer-graph workloads are disjoint.\footnote{Even though the ResNet50 DNN architecture appears in both lists, these input graphs come from different sources; the layer-graph ResNet50
    has runtimes profiled on a GPU,
    while the
    operator-graph ResNet50 has runtimes estimated for a non-GPU hardware accelerator. Thus the corresponding results are incomparable.}
    Therefore we proceed as follows:
    for each operator-graph workload,
    we manually annotate all nodes to group them into corresponding layers.
    Then we contract each layer and run the DP algorithm on the layer-graph thus obtained.
    
    We present the results of this experiment in \cref{tab:operator_vs_layer}.
    We compare the optimal contiguous splits.
    The results show that finding the best split
    on the more precise operator level
    results in a throughput advantage of up to 8\%.

\section{Experiments -- Latency Minimization}
\label{sec:latency_experiments}

In this section we evaluate our Integer Programming (IP) based algorithm
for latency minimization.
We consider the most relevant deployment scenario:
single-sample inference with memory-bound accelerators
(that is, when the entire model does not fit on one accelerator).
We run our algorithm for the same inference workloads as in \cref{sec:experiments}.
As before, we use Gurobi to solve our IP formulation.
The code and workloads used for evaluations are available at
\url{https://github.com/msr-fiddle/dnn-partitioning}.

\paragraph{Devices, implementation, experimental setup.}
We run experiments on the same inference workloads as in \cref{sec:experiments}.
However, to model a memory-bound deployment scenario where splits are necessary to fit the DNN, we assume an accelerator DRAM size of either 600 MB (for smaller DNNs, of size at most 3.6~GB) or 2~GB (for larger DNNs, of size at least 9~GB), and a number of accelerators such that the total accelerator memory is 1.4--1.8 times the size of the DNN.
Note that this implies, in particular, that a single-accelerator split is not feasible for any of our workloads.
In keeping with the mild assumption made at the beginning of \cref{sec:latency}, we assume 8 CPU cores.
We use our IP solution
to optimize for the best contiguous split.
Other implementation details
are similar
as in
\cref{sec:experiments}.

\paragraph{Baselines.}
We compare our IP algorithm against four baseline solutions.
The \textbf{first} is the following \textbf{greedy} algorithm:
\begin{itemize}
    \item contract colocated nodes and any strongly connected components that arise (as in \cref{app:preprocessing}),
    \item fix a topological ordering of the nodes,
    \item for every available accelerator, place as many nodes (in the topological order above) as will fit on the accelerator,
    \item place any remaining nodes on the CPU.
\end{itemize}
The greedy algorithm returns a contiguous split that is feasible (i.e.~satisfies the memory size constraints). For all our test workloads, it is able to place all nodes on accelerators (thus it does not use the CPU). However, it does not take processing times or communication costs into account when selecting the split.
The runtimes of this baseline are under 0.5s.

Our \textbf{second} baseline is meant to answer the following question:
\textit{If we obtain splits by optimizing the max-load objective, as we would for the throughput maximization task (that pertains to the pipelined setting), are they "good" in terms of minimizing latency as well?}
Therefore, we obtain contiguous splits by running the \textbf{max-load DP} algorithm of \cref{sec:dp}, and then we report the single-sample latency that they obtain.
The runtimes of this baseline are essentially the same as those of the max-load DP reported in \cref{sec:experiments} for the corresponding DNNs.

The \textbf{third} baseline
is \textbf{Scotch}~\cite{scotch},
a graph partitioning software used for mapping
computation graphs onto devices in a balanced way,
taking communication costs between dependent nodes
into account
(used also in \cref{sec:experiments}).
It produces non-contiguous splits.

The \textbf{fourth} baseline
are \textbf{human-expert} placements,
the same as used in \cref{sec:experiments}.

We do not compare against a local search heuristic,
as it is not clear how to design one
that satisfies the memory bounds.


\begin{table*}[t!]
\small
  \centering
  \begin{adjustbox}{width=\textwidth,center}
  \renewcommand{\arraystretch}{1.1} 
  \addtolength{\tabcolsep}{-3pt} 
  \begin{tabular}{p{1.8cm}|c||c|c|c|c||c|c|c|c}
    \multirow{2}{1cm}{\textbf{Workload}} & \multirow{2}{1cm}{\textbf{Nodes}} & \textbf{Greedy} 
    & \textbf{Max-load DP}
    & \textbf{Scotch}
    & \textbf{Expert}
    & \multicolumn{4}{c}{\textbf{IP}}
    \\
    \cline{3-10} & & \textbf{Latency} & \textbf{Latency} & \textbf{Latency} & \textbf{Latency} & \textbf{Latency} & \textbf{Runtime} & \textbf{MIP Gap} & \textbf{Gain}
    \\
    \hline
    \multicolumn{8}{l}{\small Operator-granularity graphs, single-query inference}\\
    \hline
    BERT-3 & 235 & 416.20  & 415.90 & 497.75 & - & \textbf{408.47} & 3m (6s*) & $<$1\% & 1.8\% \\
    BERT-6 & 418 & 494.13 & 445.48 & 564.61 & - & \textbf{438.06} & $>$1h (1m*) & 12.6\% & 1.7\% \\
    BERT-12 & 783 & 867.84 & 1327.03 & 1755.41 & - & \textbf{729.56} & $>$1h (10m*) & 93.8\% & 19.0\% \\
    ResNet50 & 604 & 839.54 & 1123.65 & 857.73 & - & \textbf{672.06} & $>$1h (7m*) & 54.0\% & 24.9\% \\
    \hline
    \multicolumn{8}{l}{\small Layer-granularity graphs, single-query inference}\\
    \hline
    BERT-24 & 32 & \textbf{100.22} & 108.03 & 108.03 & 111.94 & \textbf{100.22} & 14s (1s*) & $<$1\% & 0.0\% \\
    ResNet50 & 354 & 4197.06 & 1443.79 & 3610.87$^\dagger$ & OOM & \textbf{1191.02} & $>$1h (19m*) & 93.1\% & 21.2\% \\
    InceptionV3 & 652 & 2485.24 & 1621.74 & 3068.00$^\dagger$ & OOM & \textbf{1318.08} & $>$1h (43m*) & 93.5\% & 23.0\% \\
    GNMT & 192 & 268.50 & 244.33 & 636.91$^\dagger$ & 293.40$^\dagger$ & \textbf{225.6} & 3m (1m*) & $<$1\% & 8.3\% \\
  \end{tabular}
  \end{adjustbox}
  \caption{Single-sample inference workloads for latency minimization. We run the IP optimizer until it guarantees a solution within 1\% of the optimum, but no longer than 60 minutes. Where the optimization was terminated after 60 minutes, we report the optimality gap that the solver was able to certify at that time. The parenthesized times with asterisks denote the time it took the optimizer to find a solution within 2\% of the final value (though it could not yet guarantee its near-optimality). We also report the latencies obtained by the four baselines described in Section~\ref{sec:latency_experiments}; their running times are always under 0.5s (Greedy, Scotch) or the same as reported in Section~\ref{sec:experiments} (Max-load DP).
  Daggers$^\dagger$ denote a slight (between 20\% and 34\%) violation of the memory constraints,
  and "OOM" denotes a major violation (more than a factor $3 \times$).
  The best latency for each workload is given in bold. In the column "Gain" we report the advantage of our IP algorithm's solution over the best baseline.}
  \label{tab:inference-workloads}
\end{table*}

\subsection{Results}
\cref{tab:inference-workloads} shows each workload, the number of nodes (operators or layers) in the graph, and the latencies found by our IP algorithm and by the baselines. We also report running times.

As we remarked in \cref{sec:throughput}, the latency minimization task is significantly harder than throughput maximization as it contains a scheduling component.
This is reflected in the performance of our IP algorithm: for five out of eight workloads used, the IP solver did not converge to certified (near-)optimality within 1 hour.
However, it still comes out far ahead:
\begin{itemize}
    \item The IP, even where it could not prove that it has found an optimal solution, does no worse than the baselines. In fact, it outperforms the best of them by a margin of around 20\% in terms of the solution value (latency) for half of the considered workloads.
    \item Similarly as for max-load minimization (\cref{sec:experiments}), we note that the solution quality improves slowly over time, and most of the runtime is often spent on certifying the near-optimality of the found solution; it would therefore be reasonable to cut the computation much sooner, still obtaining high-quality solutions.
    \item In particular, for each workload it took the IP solver at most 7 minutes to match the solution quality of the best baseline.
\end{itemize}

\paragraph{Comparison to baselines.}
\begin{itemize}
    \item \textbf{Greedy}: our algorithm achieves latency lower by up to 72\% (i.e., over $3 \times$ faster inference; 23\% lower latency on average). 
    \item \textbf{Max-load DP}: the latency-IP achieves lower by up to 42\% (17\% on average). 
    This shows that the best splits for latency minimization
    are, indeed,
    different from the best splits for max-load minimization (throughput maximization for pipelined settings).
    Still, the max-load DP turns out to be the best baseline in 5 out of the 8 cases,
    showing some degree of compatibility between the two objectives.
    \item \textbf{Scotch}: our algorithm achieves latency lower by up to 67\% (40\% on average). In fact, Scotch never does better than the greedy heuristic. Furthermore, as it does not balance devices with respect to memory usage, it violates the memory constraints by up to 34\% in some cases.
    \item \textbf{Human expert} splits: as in \cref{sec:experiments}, we provide them for layer graphs only, due to the large node counts and high branching of operator graphs. As the expert splits were not designed with our strictly memory-bound scenario in mind, two of them are unbalanced with respect to memory usage, violating the size constraints by more than a factor $3 \times$. For the other two,
    our algorithm achieves latency lower by up to 23\% (17\% on average).
\end{itemize}

\section{Conclusions}
\label{sec:conclusions}

In this paper we give algorithms for the problem of model partitioning
of DNN workloads.
They target both inference and training, and optimize the objectives of either minimizing latency or maximizing throughput.
Our work follows a principled algorithmic approach, in which we identify the "right" combinatorial optimization problem to solve, and find \emph{provably optimal} splits.
While other approaches struggle to capture long-term dependencies in the graph
and require trying large numbers of placements on the target system, we solve the global, end-to-end joint placement and scheduling problem in one shot.
Our algorithms are efficient and can be run on arbitrary DAGs, including operator-granularity graphs, and are hardware-platform agnostic.
Experiments show that they outperform human experts and significantly improve over state-of-the-art methods.
%
%
%
%

\bibliographystyle{alpha}
\bibliography{cite}

\clearpage
\appendix

\section{Objective Functions Across Schedules}
\label{sec:max_fw_bw}

In \cref{sec:training_and_throughput_maximization}
we have argued that
for PipeDream schedules,
the objective function that accurately reflects the quality of any split,
that is, the average time taken per sample (inverse throughput),
is $\max_i (\FW{i} + \BW{i})$,
where $\FW{i}$ and $\BW{i}$
are the respective loads/costs of the forward and the backward subgraph associated with device $i$.
This is the objective function that we minimize in both our IP and DP solutions.

In the case of GPipe schedules,
we have argued that the objective function can be formulated as $\max_i \FW{i} + \max_i \BW{i}$.
This is equal to the former if the maximizing $i$'s are the same -- that is, if the bottleneck device is the same for the pipelined forward pass
(the first seven columns in \cref{fig:train_gpipe})
as for the pipelined backward pass
(the next seven columns).

This usually holds true for real-world DNN workloads due to three factors described below:
\begin{itemize}
    \item For any device, its forward subgraph $S$ and its backward subgraph $S'$ contain paired nodes; that is, most nodes in the backward subgraph $S'$ have a corresponding forward node, which, due to colocation constraints, will be in $S$, and vice versa.
    For instance,
    most forward nodes operate on a set of weights,
    for which the backward pass then computes gradients and weight updates.
    \item The processing and communication times of such corresponding/colocated nodes are correlated; for example, if the forward node corresponds to a matrix multiplication, then the processing times of both forward and backward nodes will grow with the size of the matrix.
    \item In fact, GPipe
    uses a re-materialization technique~\cite{Chen2016TrainingDN}
    to save memory: it discards stashed activations generated in the forward pass (needed later in the backward pass), and instead reruns the forward pass operators in the backward pass to re-materialize the required stashed activations for the backward operators.
    If this is reflected in the DNN workload operator-graph or layer-graph, then it further increases the aforementioned correlation between forward and backward times.
\end{itemize}

\begin{figure}[t!]
\centering
\includegraphics[width=0.65\textwidth]{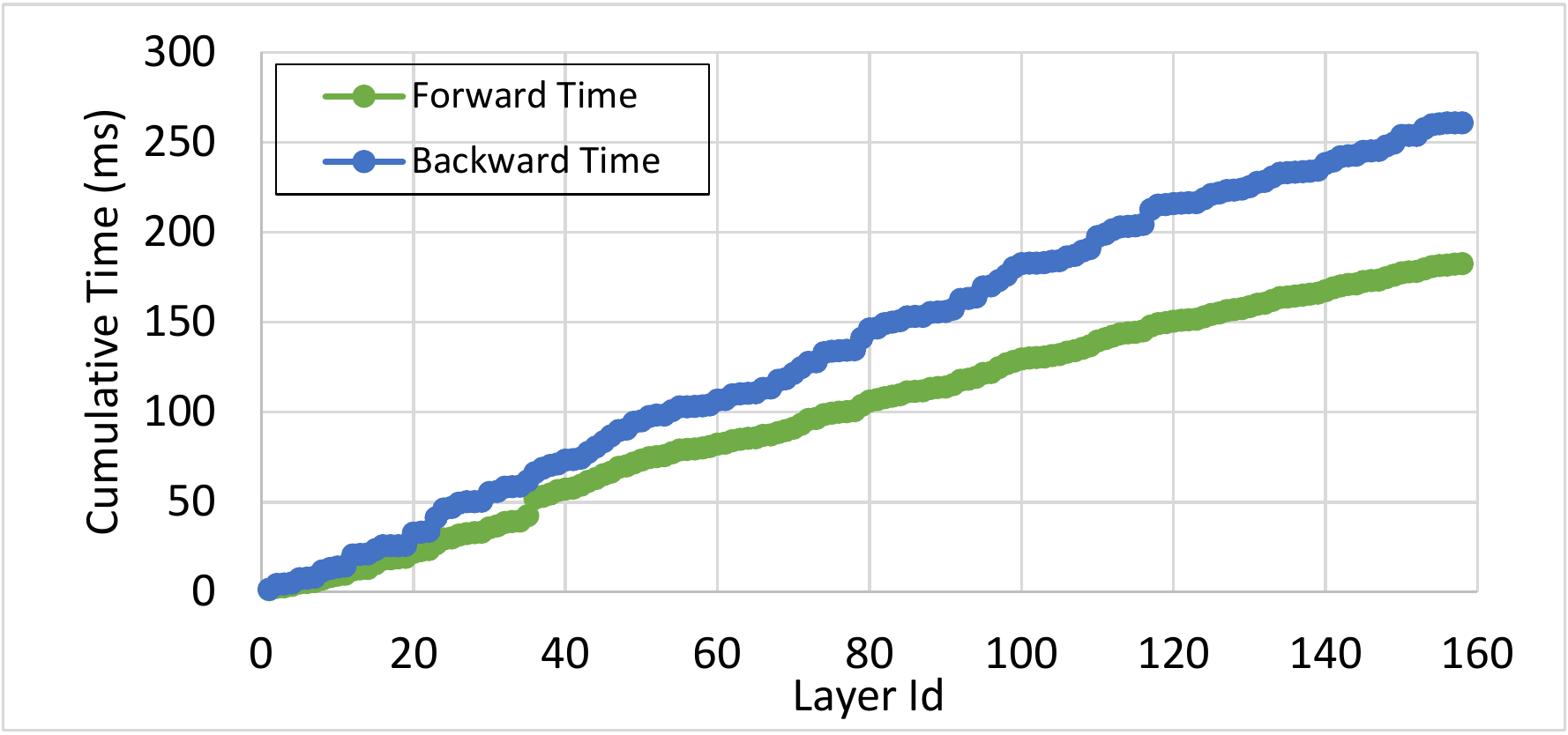}
\caption{Cumulative training time for forward and backward layers of ResNet50 (layer graph). The time accumulates with each layer progressively, that is, the $i$-th entry is the sum of processing times of layers from $1$ to $i$.}
\label{fig:cdf}
\vspace{-7pt}
\end{figure}

In \cref{fig:cdf} we plot cumulative forward and backward times for an example training workload (that does not use re-materialization), which grow at a similar pace.
These runtimes have been profiled on a GPU.

The above discussion motivates the use of our objective $\max_i (\FW{i} + \BW{i})$ as a proxy for the objective $\max_i \FW{i} + \max_i \BW{i}$
also in the case of GPipe schedules.
Nonetheless, our IP solution can also be adjusted to optimize the latter objective.
Unsurprisingly, we empirically find that splits found by optimizing either objective differ by at most 6\% when using re-materialization.

\section{DP Preprocessing and Reductions}
\label{app:preprocessing}

\paragraph{DP preprocessing.}
In our Dynamic Programming solution we need to handle colocation constraints given in the input: certain pairs of nodes operate on the same state and thus they are required to be on the same device.
A common scenario where this arises concerns forward and backward nodes that operate on the same set of weights, but pairs of forward nodes (or pairs of backward nodes) can also be colocated.
In the input files this is expressed via the \texttt{colorClass} field: nodes of the same color class must be placed on the same device.

Moreover, for training workloads, the DP can natively find only contiguous splits, but now most devices need to be assigned two contiguous subgraphs (backward and forward). 
Therefore we run the DP only on the forward part, but we take the corresponding backward nodes together with every considered contiguous subgraph.
However, some care is required to make sure that we assign those backward nodes that do not have a corresponding forward node; we call these backward nodes \emph{orphaned}.

\newcommand{\cfw}{\ensuremath{C_{\textrm{FW}}}}
\newcommand{\cbw}{\ensuremath{C_{\textrm{BW}}}}
For the reasons outlined above, our solution needs to run a series of preprocessing steps before the core DP method can be applied:
\begin{itemize}
    \item For every color class $C \subseteq V$, i.e., a set of nodes that must be colocated, let $\cfw$ and $\cbw$ be the forward and backward nodes in $C$, respectively (so that $C = \cfw \cup \cbw$).
    We \emph{contract} each set $\cfw$ and each set $\cbw$ (that is, we compress each of them into a single node; this new node will be forward for $\cfw$ and backward for $\cbw$).
    \item The input graph was guaranteed to be acyclic at the beginning,
    but the new contracted graph may no longer be acyclic.
    For instance, there could be a path $u, v, w$ where $u$ and $w$ are colocated (but not $v$); then the contracted graph will have edges in both directions between $v$ and the new node corresponding to $\{u,w\}$. In the original graph, any colocation-respecting contiguous split would need to contain all of $u,v,w$ in a single subgraph;
    more generally, every strongly connected component in the contracted graph needs to be colocated. Therefore, we contract all strongly connected components. Now the contracted graph is again acyclic.
    \item Later, when we run the DP, while considering a subgraph $S$ of forward nodes we will consider the subgraph $S'$ of their corresponding backward nodes at the same time, and take the total computation and communication cost of $S \cup S'$ into account. Thus, when we have assigned all forward nodes, we will have also assigned all backward nodes that are not orphaned. However, orphaned nodes would not be assigned to any subgraph/device.
    
    To prevent this behavior, we introduce new artificial forward nodes, to be images of the orphaned backward nodes. When the DP decides where to place these new forward nodes, it will also have decided the placement of the orphaned backward nodes. (At the end we will remove the artificial nodes from the final split.)
    
    However, if the new nodes are isolated (have no adjacent edges), then the number of ideals grows exponentially\footnote{Suppose we have introduced $r$ such new nodes; since each of them is free to be or not be in an ideal, the number of ideals grows by a factor $2^r$, and the DP runtime, which depends on the number of ideal pairs $I', I$ with $I' \subseteq I$, grows by a factor $3^r$.}; furthermore, as the placement of the new forward and orphaned backward nodes is arbitrary, we may end up with non-contiguous splits on the backward side.
    
    To deal with these issues, we also add new artificial edges adjacent to the new artificial nodes. Since backward nodes and edges mostly resemble a mirror image of their corresponding forward nodes and edges, we add the new edges in such a way as to also build such a mirror image. Namely, for a backward edge $(u',v')$ where at least one of $u'$, $v'$ is orphaned, we add a forward edge $(v,u)$, where $u$ and $v$ are the forward images of $u'$ and $v'$ respectively (note that at least one of $u$, $v$ is a new artificial node).
\end{itemize}

After these preprocessing steps, we can use our core DP method on the contracted graph.
Once this is done, we map the resulting splits back to the original graph and return the result.
For more details on implementation, see the attached code and the comments therein.

We remark that due to our preprocessing steps,
the number of ideals may sometimes be smaller
than the number of nodes in the initial input graph
(this happens for several of our workloads in \cref{tab:tput-workloads}).

\paragraph{Non-uniform outgoing communication costs.}
In the case of operator graphs,
the input files for our solvers are created based on ONNX computation graphs.
There, communication costs are given on edges, rather than on nodes as we require in our model (see \cref{sec:model}).
In the vast majority of cases, all edges going out of the same node $u$ have the same cost, and we can set that cost as parameter $\comm{u}$.
However, sometimes there could be two or more edges with different costs going out of the same node in an ONNX graph;
this situation corresponds to e.g.~sending different parts of the operator's output to different operators.
In this case, we perform the following reduction:

    Suppose that $u$ has outgoing edges to nodes $v_1, v_2, ..., v_r$ with possibly different edge costs $d_1, d_2, ..., d_r$. For every outgoing edge $(u,v_j)$, we \emph{subdivide} it:
    insert a new node $w_j$ in the middle and replace the edge $(u,v_j)$ with two edges $(u,w_j)$ and $(w_j,v_j)$. The new node $w_j$ should have $\cpu{w_j} = \fpga{w_j} = \mem{w_j} = 0$ and be colocated with $u$.
    We set $\comm{w_j} = d_j$.
    Finally, set $\comm{u}$ to any value, say $\infty$; this communication cost will never be paid in any feasible solution, as now $u$ is colocated with all of its successors, which are $w_1, w_2, ..., w_r$.
    
    After obtaining a final split, we may remove the artificial nodes $w_j$ from the solution.
    It is easy to see that the way we have set the outgoing communication costs $\comm{}$ on nodes reflects the edge-communication costs given in the input ONNX graph.

\section{Extensions}
\label{sec:extensions}

In this section,
which deals with throughput maximization
(i.e.~the pipelined setting),
we briefly explain how to adjust our model and solutions
so as to account for certain different or more complex deployment scenarios
that appear in related work or in practice.

\subsection{Interleaving Communication and Computation}
\label{sec:interleaving_communication_and_computation}

Throughout the paper we have assumed that accelerators are invoked when their inputs are ready, at which point they are transferred to the accelerator memory; next, computation takes place; next, outgoing transfers take place
(see \cref{sec:model}).
After that, the in-transfer for the subsequent sample/minibatch may begin, and so on.
For this reason, the load of a device is defined as the \emph{sum}
of the computation cost and the communication cost.
However, it is also reasonable to assume that communication (data transfers) may proceed in parallel to computation, at least for different samples.
For instance, once we have finished the in-transfer for sample 1, we might simultaneously start the processing of sample 1 and the in-transfer for sample 2.
This is the setting considered in the PipeDream paper~\cite{harlap2018pipedream}.

Both of our solutions (DP and IP) can be easily adjusted to this setting:
one just needs to define the load of a device
as the \emph{maximum} of the computation cost and the communication cost,
rather than the sum.
In terms of pipeline schedules,
one can think of
splitting the device into two virtual devices,
one holding the communication portion of the load
and
the other holding the computation portion,
that \emph{can} be processing at the same time.
Then either virtual device could
be a bottleneck in the pipeline.

In fact, one can further assume that the in-transfer and the out-transfer are done over separate channels (full-duplex communication);
then a maximum of three quantities (in-transfer cost, computation, out-transfer cost) should be used.

\subsection{Replication}

An alternative to model parallelism is data parallelism:
an approach where the entire model is replicated over multiple devices
that process minibatches in parallel.
When using this approach,
the communication cost associated with synchronizing the parameters of the entire model proves to be very high for many DNN workloads.
Nevertheless, it can also yield large gains for other workloads, especially sparser ones (with a small number of parameters relative to computation cost).
PipeDream~\cite{harlap2018pipedream} proposed 
a hybrid model-parallel/data-parallel approach,
where we form a pipeline, but certain subgraphs in this pipeline
can be \emph{replicated} over multiple devices.
This allows the automated partitioner to replicate those fragments of the network that will reap the most benefit
while keeping synchronization costs low.

We can also introduce this capability into our DP algorithm.
When the DP decides whether to place the currently considered subgraph on a CPU or on an accelerator, now it will also decide how many devices to use. That is, in the DP relation, where previously we had $\max\left( \dpstate{k'-1}{\ell'}{I'}, \fpganoindex(I \setminus I') \right)$,
now we
write\footnote{We treat the CPU-related term similarly.}
\[\min\nolimits_{k''=1}^{k'}\max\left( \dpstate{k'-k''}{\ell'}{I'}, \fpganoindex(I \setminus I', k'') \right),\]
where $\fpganoindex(I \setminus I', k'')$
is the average time per sample for this subgraph when replicated over $k''$ accelerators.
In absence of weight synchronization,
this average time would be just $\fpganoindex(I \setminus I') / k''$.
Weight synchronization (assuming efficient AllReduce
collective communication) contributes a term $\left((k'' - 1) \cdot \sum_{v \in I \setminus I'} \mem{v}\right) / (k'' \cdot B)$,
where $\mem{v}$ are sizes of weights associated with nodes
and
$B$ is the communication bandwidth.
Thus, $\fpganoindex(I \setminus I', k'')$ should be either the sum or maximum of these two terms, depending on our assumption of interleaving communication with computation (see \cref{sec:interleaving_communication_and_computation}).

This modification of the DP increases the running time by a factor of $O(\max(k,\ell))$. The memory usage remains unchanged.

\subsection{Accelerator Hierarchies}

Throughout the paper we have assumed 
a homogeneous system
with $k$ accelerators
and $\ell$ CPU cores
(probably a single machine).
To more precisely capture a distributed setting,
one can consider a hierarchical collection of accelerators,
such as clusters of GPUs connected internally with faster interconnects and externally (i.e.~between clusters) with slower connections (or over a network).
Such a multi-level model is used in
PipeDream~\cite{harlap2018pipedream}.
Now, the cost of transferring data over an edge between two nodes
depends on whether these nodes are placed on devices in the same or different clusters (or even on different machines).
The main new challenge is knowing which cost should be taken into account.

The DP solution in PipeDream handles this by dynamically computing optimal splits not only for prefixes of the input network
(that correspond to our ideals),
but for every contiguous segment of the network.
We remark that we can use the same method, at a cost of an $O(\numideals)$-factor increase in both memory usage (number of DP states) and running time.

\end{document}